\documentclass{article}

\PassOptionsToPackage{numbers, compress, sort}{natbib}
\usepackage[final]{neurips_2022}

\usepackage[utf8]{inputenc} %
\usepackage[T1]{fontenc}    %
\usepackage[hidelinks]{hyperref}       %
\usepackage{url}            %
\usepackage{booktabs}       %
\usepackage{amsfonts}       %
\usepackage{nicefrac}       %
\usepackage{microtype}      %
\usepackage{xcolor}         %

\usepackage{amsmath}
\usepackage{bm,dsfont,diagbox,amsthm,xspace}
\usepackage{enumitem}
\usepackage{amsthm}      
\usepackage{dsfont}
\usepackage{amsmath}
\usepackage{graphicx}
\usepackage[normalem]{ulem}
\usepackage{subcaption}
\usepackage{multirow}
\usepackage{natbib}
\usepackage{algorithm}
\usepackage{algorithmic}
\usepackage{thmtools}
\usepackage{thm-restate}
\usepackage{bbm}
\usepackage{makecell}
\usepackage{bibunits}
\usepackage{diagbox} \usepackage{cleveref}
\usepackage{wrapfig}
\usepackage{thm-restate}

\usepackage{bm,xspace}

\newcommand{\update}[1]{{\color{black}#1}}

\newcommand*\dbar[1]{\overline{\overline{\lower0.2ex\hbox{$#1$}}}}

\usepackage{tikz}
\usetikzlibrary{shapes.misc}
\colorlet{darkgreen}{green!50!black}

\newcommand{\harrow}[1]{\mathstrut\mkern2.5mu#1\mkern-11mu\raise1.6ex%
  \hbox{$\scriptscriptstyle\rightharpoonup$}}

\usepackage{amsmath,amsfonts,bm,dsfont,diagbox,amsthm}

\newcommand{\CRupdate}[1]{{\color{black}#1}}
\newcommand{\Ind}{\mathds{1}}

\newcommand{\tr}{\text{tr}}
\newcommand{\te}{\text{te}}

\newcommand{\eat}[1]{}

\makeatletter
\let\save@mathaccent\mathaccent
\newcommand*\if@single[3]{%
    \setbox0\hbox{${\mathaccent"0362{#1}}^H$}%
    \setbox2\hbox{${\mathaccent"0362{\kern0pt#1}}^H$}%
    \ifdim\ht0=\ht2 #3\else #2\fi
    }
\newcommand*\rel@kern[1]{\kern#1\dimexpr\macc@kerna}
\newcommand*\widebar[1]{{\@ifnextchar^{{\wide@bar{#1}{0}}}{\wide@bar{#1}{1}}}}
\newcommand*\wide@bar[2]{\if@single{#1}{\wide@bar@{#1}{#2}{1}}{\wide@bar@{#1}{#2}{2}}}
\newcommand*\wide@bar@[3]{%
\begingroup
\def\mathaccent##1##2{%
    \let\mathaccent\save@mathaccent
    \if#32 \let\macc@nucleus\first@char \fi
    \setbox\z@\hbox{$\macc@style{\macc@nucleus}_{}$}%
    \setbox\tw@\hbox{$\macc@style{\macc@nucleus}{}_{}$}%
    \dimen@\wd\tw@
    \advance\dimen@-\wd\z@
    \divide\dimen@ 3
    \@tempdima\wd\tw@
    \advance\@tempdima-\scriptspace
    \divide\@tempdima 10
    \advance\dimen@-\@tempdima
    \ifdim\dimen@>\z@ \dimen@0pt\fi
    \rel@kern{0.6}\kern-\dimen@
    \if#31
        \overline{\rel@kern{-0.6}\kern\dimen@\macc@nucleus\rel@kern{0.4}\kern\dimen@}%
        \advance\dimen@0.4\dimexpr\macc@kerna
        \let\final@kern#2%
        \ifdim\dimen@<\z@ \let\final@kern1\fi
        \if\final@kern1 \kern-\dimen@\fi
    \else
        \overline{\rel@kern{-0.6}\kern\dimen@#1}%
    \fi
}%
\macc@depth\@ne
\let\math@bgroup\@empty \let\math@egroup\macc@set@skewchar
\mathsurround\z@ \frozen@everymath{\mathgroup\macc@group\relax}%
\macc@set@skewchar\relax
\let\mathaccentV\macc@nested@a
\if#31
    \macc@nested@a\relax111{#1}%
\else
    \def\gobble@till@marker##1\endmarker{}%
    \futurelet\first@char\gobble@till@marker#1\endmarker
    \ifcat\noexpand\first@char A\else
        \def\first@char{}%
    \fi
    \macc@nested@a\relax111{\first@char}%
\fi
\endgroup
}
\makeatother

\newtheorem{definition}{Definition}
\newtheorem{proposition}{Proposition}

\newtheorem{corollary}{Corollary}
\newtheorem{lemma}{Lemma}

\newtheorem{assumption}{Assumption}
\def\eqref#1{equation~\ref{#1}}

\def\1{\bm{1}}

\def\mA{{\bm{A}}}
\def\mB{{\bm{B}}}

\def\mF{{\bm{F}}}

\def\mP{{\bm{P}}}

\def\mS{{\bm{S}}}

\def\mW{{\bm{W}}}

\DeclareMathAlphabet{\mathsfit}{\encodingdefault}{\sfdefault}{m}{sl}
\SetMathAlphabet{\mathsfit}{bold}{\encodingdefault}{\sfdefault}{bx}{n}

\def\cB{{\mathcal{B}}}

\def\cG{{\mathcal{G}}}

\def\cX{{\mathcal{X}}}

\def\sN{{\mathbb{N}}}

\def\sR{{\mathbb{R}}}
\def\sS{{\mathbb{S}}}

\DeclareMathOperator*{\argmax}{arg\,max}

\def\L2L{L^2(\mathcal{J}) \rightarrow L^2(\mathcal{J})}

\def\cmin{\mathrm{d}_{\mathrm{min}}}

\def\LipPhi{L_{\Phi}}
\def\LipPsi{L_{\Psi}}

\renewcommand{\P}{\mu}

\renewcommand{\L}{\mathcal{L}} %

\newcommand{\onenode}{{\bullet}}
\newcommand{\pairwisenodes}{{\bullet \bullet}}

\newcommand{\etaone}{\eta^{\onenode}}
\newcommand{\etatwo}{\eta^{\pairwisenodes }}

\newcommand{\firstlayerd}{{F_0}}
\newcommand{\vfnull}{{\bf f}}

\newcommand{\vfonet}{{\bf f}^{ {\onenode{(t)}} }}
\newcommand{\vfonetmone}{{\bf f}^{ {\onenode{(t-1)}} }}

\newcommand{\Mtwo}{{M}^{\pairwisenodes}}
\newcommand{\Mone}{M^{\onenode}}

\newcommand{\gonet}{g^{\onenode{(t)} }}

\newcommand{\fone}{f^{\onenode}}
\newcommand{\fonebar}{\widebar{f}^{\onenode}}
\newcommand{\ftwo}{f^{\pairwisenodes }}
\newcommand{\fonet}{f^{\onenode{(t)} }}
\newcommand{\fonezero}{f^{\onenode{(0)} }}
\newcommand{\fonebart}{\widebar{f}^{\onenode{(t)} }}
\newcommand{\fonetmone}{f^{\onenode{(t-1)} }}

\newcommand{\Thetatwo}{\Theta^{\pairwisenodes }}
\newcommand{\Thetatwot}{\Theta^{\pairwisenodes{(t)} }}
\newcommand{\ThetatwoTcap}{\Theta^{\pairwisenodes{(T)} }}
\newcommand{\Thetaone}{\Theta^{\onenode}}

\newcommand{\Thetaonet}{\Theta^{\onenode{(t)} }}

\newcommand{\ThetaoneTcap}{\Theta^{\onenode{(T)} }}

\newcommand{\ThetaonebarTcap}{\widebar{\Theta}^{\onenode{(T)} }}

\newcommand{\MPNN}{MPNN\xspace}

\newcommand{\gMPNNs}{gMPNNs\xspace}
\newcommand{\gMPNN}{gMPNN\xspace}
\newcommand{\cMPNNs}{cMPNNs\xspace}
\newcommand{\cMPNN}{cMPNN\xspace}
\newcommand{\onegMPNN}{$\text{gMPNN}^\onenode$\xspace}
\newcommand{\onecMPNN}{$\text{cMPNN}^\onenode$\xspace}
\newcommand{\twogMPNN}{$\text{gMPNN}^\pairwisenodes$\xspace}
\newcommand{\twocMPNN}{$\text{cMPNN}^\pairwisenodes$\xspace}
\newcommand{\onegMPNNs}{$\text{gMPNNs}^\onenode$\xspace}
\newcommand{\onecMPNNs}{$\text{cMPNNs}^\onenode$\xspace}
\newcommand{\twogMPNNs}{$\text{gMPNNs}^\pairwisenodes$\xspace}

\newcommand{\deltaAWone}{\delta^{\onenode}_\text{A-W}}
\newcommand{\deltaAWtwo}{\delta^{\pairwisenodes}_\text{A-W}}

\newcommand{\cdW}{{c_W}}
\newcommand{\cdA}{{c_A}}

\newcommand{\SBM}{SBM\xspace}
\newcommand{\SBMs}{SBMs\xspace}

\newcommand{\mytitle}[1]{OOD Link Prediction Generalization Capabilities#1 of Message-Passing GNNs in Larger Test Graphs}
\title{\mytitle{\\}}

\author{%
  Yangze Zhou \\
  Department of Statistics\\
  Purdue University\\
  West Lafayette, IN 47903 \\
  \texttt{zhou950@purdue.edu} \\
   \And
   Gitta Kutyniok\\
   Department of Mathematics\\ Ludwig-Maximilians-Universitat M\"unchen\\
   Munich, Germany \\
   \texttt{kutyniok@math.lmu.de} \\
   \And
   Bruno Ribeiro\\
   Department of Computer Science\\
   Purdue University\\
   West Lafayette, IN 47903 \\
   \texttt{ribeiro@cs.purdue.edu} 
}

\begin{document}

\maketitle

\begin{abstract}
This work provides the first theoretical study on the ability of graph Message Passing Neural Networks (\gMPNNs) ---such as Graph Neural Networks (GNNs)--- to perform inductive out-of-distribution (OOD) link prediction tasks, where deployment (test) graph sizes are larger than training graphs. We first prove non-asymptotic bounds showing that link predictors based on permutation-equivariant (structural) node embeddings obtained by \gMPNNs can converge to a random guess as test graphs get larger. We then propose a theoretically-sound \gMPNN that outputs structural pairwise (2-node) embeddings and prove non-asymptotic bounds showing that, as test graphs grow, these embeddings converge to embeddings of a continuous function that retains its ability to predict links OOD. Empirical results on random graphs show agreement with our theoretical results.
\end{abstract}

\begin{bibunit}[plainnat]

\section{Introduction}
Link prediction is the task of predicting whether two nodes likely have a missing link~\citep{adamic2003friends,de2008logical,liben2007link,koller2007introduction,taskar2003link}. 
Link prediction tasks arise in many settings, ranging from predicting edges on bipartite graphs between users and products or content in recommender systems~\citep{bell2007lessons,das2007google,koren2009matrix,koren2022advances,linden2003amazon,smith2017two}, to knowledge graph reconstruction~\citep{angeli2013philosophers,taskar2003link,getoor2005link,nickel2015review,trouillon2016complex,dettmers2018convolutional}, to predicting protein-protein interactions~\citep{qi2006evaluation}.

In recent years, there has been growing interest in applying neural network models to inductive link prediction tasks.
Inductive link prediction considers methods trained on a graph $G^\tr$ and deployed at test time on another graph $G^\te$.
It also encompasses the task of training the method on a smaller induced subgraph $G^\tr$ of a larger graph $G^\te$, then deploying it on the entire graph.
In particular, our work focuses on graph message-passing Neural Networks (\gMPNNs)~\citep{gilmer17a,santoro2017simple} or, more precisely, the widely used Graph Neural Network (GNN) framework~\citep{sperduti1997supervised,gori2005new,scarselli2008graph,bruna2013spectral,defferrard2016convolutional,Kipf2016,hamilton2017inductive,velivckovic2017graph,bronstein2017geometric}.

Our work asks the following questions: {\em Are link prediction methods able to cope with the task of inductive out-of-distribution (OOD) link prediction,
where (unseen) test graphs are significantly larger than training graphs?}
How can these OOD link prediction tasks be theoretically defined?
Can we obtain non-asymptotic bounds on the generalization capabilities of these methods?

The majority of today's link prediction methods are based on a similar principle. 
Consider an attributed graph $G=(V,E)$, with node set $V=\{1,...,N\}$, edge set $E \subseteq V \times V$, and node features $\mF \in \sR^{N\times F_0}$, $F_0 \geq 1$. 
Then, given a pair of nodes $i,j \in V$, after $T \geq 1$ iterations over $G$, these methods produce associated node embeddings (representation vectors) $\Thetaone_i, \Thetaone_j \in \sR^{F_T}$, $F_T \geq 1$, which are then used in a link function $\etaone:\sR^{F_T}\times \sR^{F_T} \to [0,1]$ such that $\etaone(\Thetaone_i,\Thetaone_j)$ predicts the probability that $i$ and $j$ have a missing link in $G$. 
In our notation we will denote all node embeddings and associated functions with the superscript ``$\!~^{\bullet}$''.
Henceforth we denote \gMPNNs that output structural node embeddings as \onegMPNNs.

{\em Node embeddings.} The first part of our work considers a subset of these methods, where the output node embeddings are permutation equivariant (a.k.a.\ {\em structural node embeddings}~\citep{srinivasan2019equivalence}).
Informally, a sequence of node embeddings $\Thetaone \in \sR^{N \times F_T}$ given by an embedding method is permutation-equivariant if for any arbitrary graph $G$ and any permutation $\pi \in \sS_N$ of the node indices, where $\sS_N$ is the symmetric group, the resulting isomorphic graph $G'=(\pi \circ V,\pi \circ E,\pi \circ \mF)$ gets permuted node embeddings ${\Thetaone}' = \pi \circ \Thetaone$, where $\pi \circ M$ defines the action of $\pi$ on $M$ (we will provide a formal definition in \Cref{sec:prelim}).
We leave the study of OOD link prediction with {\em positional node embeddings} (a.k.a.\ permutation-sensitive node embeddings~\citep{srinivasan2019equivalence}) to future work.

The application of GNNs to link prediction tasks is made difficult by the fact that, by construction, permutation-equivariant GNNs give the same embeddings $\Thetaone_i,\Thetaone_j$ to any isomorphic nodes $i,j$ in $G$, as noted by~\citet{you2019position} and \citet{srinivasan2019equivalence}.
Isomorphic nodes are nodes that are structurally indistinguishable in $G$ (even when considering node features) except by their (assumed arbitrary) node indices $i,j \in V$.
\update{
That is, if a graph has isomorphic pairs, permutation-equivariant GNN link prediction can fail.
The recent link prediction literature has significantly relied on isomorphic nodes for theoretical results (e.g., \citet[Theorem 2]{zhang2021labeling} uses isomorphic nodes to prove that, uniformly, graphs are likely to have many isomorphic nodes and hence are not amenable to accurate link prediction). 
However, isomorphic nodes are rare in both real-world graphs (see \Cref{fig:isonodes} in the Appendix) and in large random graphs (\Cref{prop:noniso}).

An important open question is whether equivariant GNN would be able to predict links in asymmetric graphs. That is, the concerns of~\cite{you2019position,srinivasan2019equivalence} may not be of practical importance. 
Our work also answers this question: We see that for in-distribution link prediction tasks (where graph test sizes are the same as training sizes), permutation-equivariant GNNs are able to predict links by tapping into the graph asymmetries.
However, we show theoretically and empirically that tapping into asymmetries can fail OOD even when it works in-distribution.
}

{\em Pairwise embeddings.} Taking a different route, \citet{srinivasan2019equivalence} \update{provides an existence proof} that the link prediction task between $i$ and $j$ can always be performed by a pairwise embedding $\Thetatwo_{ij}(G)$, i.e., for any pair of nodes $i,j$ in a graph $G$, there exists a pairwise embedding $\Thetatwo_{ij}(G)$ and a link function $\etatwo:\sR^{F_T} \to [0,1]$ such that $\etatwo(\Thetatwo_{ij})$ approximates the probability that $i$ and $j$ have a hidden link.
In our notation we will denote all pairwise (joint 2-node) embeddings and associated functions with the superscript ``$\!~^{\bullet\bullet}$''.
\update{Unfortunately, as the test graph grows, we were unable to prove existing pairwise embedding methods~\cite{monti2018dual,wang2021inductive,Zhang2017WLlink,zhang2021labeling,zhu2021neural} are able to perform OOD link prediction tasks.
Hence, we propose a novel family of \gMPNNs for pairwise embeddings, denoted \twogMPNNs henceforth.}
The second part of our work considers the OOD generalization capability of these \twogMPNNs.

\vspace{-7pt}
\paragraph{Contributions.} In this work we study inductive OOD link prediction tasks for larger test graphs using permutation-equivariant node and pairwise embeddings, $\Thetaone$ and $\Thetatwo$, respectively.
Our work makes the following contributions:
\begin{enumerate}[leftmargin=*]%
    \item We provide a theoretical framework defining OOD inductive link prediction tasks, where test graphs are significantly larger than training graphs.
    \item \update{We show that structural node embeddings from message-passing GNNs can fail in OOD link prediction tasks if the test graph (from the same graph family) is significantly larger than the training graph.}  Our work fills {\em an important gap in the literature}, where \citet{DBLP:conf/icml/BevilacquaZ021} studied the OOD capabilities of GNNs for {\em graph classification} using random graph models. Our work studies the OOD capabilities of GNNs for {\em inductive link prediction} in a similar setting.
    \item \update{We propose a new family of structural pairwise embeddings, denoted \twogMPNNs, that can provably perform the above OOD task.}
    \item \update{We provide non-asymptotic bounds on the convergence of pairwise \gMPNNs embeddings}.
    Extensive empirical experiments using stochastic block models (SBMs~\citep{snijders1997estimation}) validate our theoretical results.
    Our work focuses on providing a theoretical understanding of the challenges of OOD link prediction tasks rather than propose real-world link prediction tasks and compare baselines. However, we believe that our work lays the theoretical foundation (and challenges) for future application-focused works. 
\end{enumerate}
\vspace{-7pt}
\section{Preliminaries}
\label{sec:prelim}
\vspace{-5pt}

Given an attributed graph $G=(V,E)$, with node set $V=\{1,...,N\}$, edge set $E \subseteq V \times V$, adjacency matrix $\mA \in \{0,1\}^{N\times N}$, where $\mA_{ij} = \Ind_{\{(i,j) \in E\}}$, and node features $\mF \in \sR^{N\times \firstlayerd}$, $\firstlayerd > 0$.
Let $\mP_\pi \in \cB_N$ be a permutation matrix associated with permutation $\pi \in \sS_N$ (where $\sS_N$ is the symmetric group), where $\cB_N$ denotes the Birkhoff polytope of $N \times N$ doubly-stochastic matrices. Doubly-sctochastic matrices are non-negative square matrices whose rows and columns sum to one.
The matrix $\mP_\pi$ defines the action of permutation $\pi$ on these matrices, e.g., $\pi \circ \mA = \mP_\pi \mA \mP_\pi^T$.
We denote a pair of nodes $i,j \in V$ as isomorphic in $G$ if exists $\pi \in \sS_N$ such that \update{$\pi_i = j$, $\mA = \mP_\pi \mA \mP_\pi^T$,} and $\mF = \mP_\pi \mF$.
Node features can be defined by the {\em graph signal} $f:V\rightarrow \sR^\firstlayerd$ as a function that maps a node to an $\firstlayerd$-dimensional feature in $\sR^\firstlayerd$. Then the signal of the graph $\mF$ can be represented by a matrix $\mF = [\vfnull_1,...,\vfnull_N]^T\in \sR^{N\times \firstlayerd}$, where $\vfnull_{i}\in \sR^\firstlayerd$ are the features of node $i \in V$.

{\bf Random graph model for $G$.}
Denote the metric-measure space by $(\cX,d,\mu)$, where $\cX$ is a set, $d$ is a metric, and $\mu$ is a probability Borel measure. A \emph{graphon} is defined as a mapping $W: \cX\times \cX\rightarrow [0,1]$~\citep{diaconis2007graph,wolfe2013nonparametric}. 
In what follows we define how the graph $G$ is sampled from the graphon models. The signal definition follows \update{\citet[Definition 2.3]{maskey2022generalization}} and the edge samples follow \citet{lawrence2005probabilistic,airoldi2013stochastic}.
\begin{wrapfigure}{r}{0.33\textwidth}
\begin{center}
\vspace{-8pt}
\hspace{-15pt}\includegraphics[width=2in]{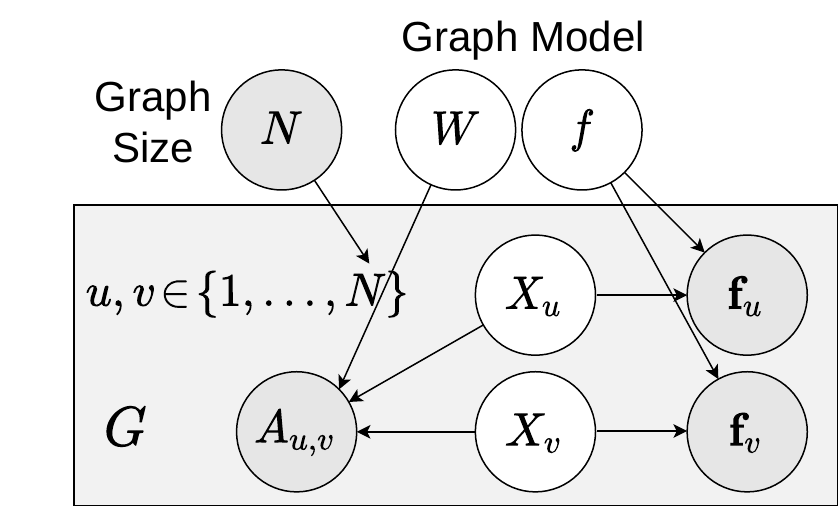}
\vspace{-10pt}
\caption{Templated causal DAG of $G$\label{fig:DAG}. Hidden and observed variables are shaded white and gray, respectively.}
\vspace{-20pt}
\end{center}
\end{wrapfigure}
\begin{definition}[Random graph model]
\label{def:rgm}
We define $(W,f)$ as a {\em random graph model} for $G$ on $(\cX,d,\mu)$ with the graphon $W: \cX\times \cX\rightarrow [0,1]$ and the metric-space signal $f:\cX\rightarrow \sR^\firstlayerd$. $f \in L^\infty(\cX)$ is an essentially bounded measurable function with the essential supremum norm. We obtain $(G,\mF)$ by first sampling $N$ i.i.d.\ random points $X_1,...,X_N$ from $\cX$ with probability density $\mu$, as the nodes of $G$.
Then the edge $(i,j)$ between nodes $i$ and $j$ is sampled with probability $W(X_i,X_j)$, i.e., the adjacency matrix $\mA=(\mA_{i,j})_{i,j}$ of $G$ is defined as $\mA_{i,j} = \mathbbm{1}(Z_{i,j}<W(X_i,W_j))$ for $i,j=1,...,N$, where $\{Z_{i,j}\}_{i,j=1}^N$ are sampled i.i.d.\ from $\text{Uniform}(0,1)$.
The graph signal $\mF = [\vfnull_1,...,\vfnull_N]^T \in \sR^{N \times \firstlayerd}$ is defined as $\vfnull_i=f(X_i)$. We say $(G,\mF)\sim (W,f)$.
Further, we restrict our attention to graphons $W$ such that there exists a constant $d_\text{min}$ satisfying the {\em graphon degree} $d_W(x):=\int_{\mathcal{X}}W(x,y)d\mu(y)\geq d_\text{min}>0,\forall x\in\mathcal{X}$.
\end{definition}
In an abuse of notation we identify node $i \in V$ with the sampled value $X_i \sim \mu, \forall i\in \{1,...,N\}$, since generally $\mu$ is such that $P(X_i = X_j)=0$ almost everywhere for $i\neq j$ (e.g., $\mu$ is uniform).
The causal DAG of the data generation process of $G$ is given in \Cref{fig:DAG}.
\update{Our goal is to produce predictors that survive the distribution shift implied by a change in the distribution of graph sizes $N$ during test time.}
Furthermore, we note that all proofs are relegated to the Appendix due to space constraints.

\newcommand{\doop}{\text{do}}
\newcommand{\supp}{\text{supp}}

\vspace{-6pt}
\subsection{Inductive structural node representations with graph message-passing neural networks}
\vspace{-5pt}
Henceforth use the terms {\em node embeddings} and {\em node representations} interchangeably.
Graph message-passing Neural Network (\onegMPNN) is defined by realizing a message-passing Neural Network (\MPNN) on a graph.
\update{We now restate the \citet[Definition 2.1]{maskey2022generalization} of MPNN.} %

\begin{definition}[\MPNN \citep{maskey2022generalization}]
\label{def:MPNN}
Let $T\in \sN$ denote the number of layers. For $t=1,...,T$, let $\Phi^{(t)}:\sR^{2F_{t-1}}\rightarrow \sR^{H_{t-1}}$ and $\Psi^{(t)}:\sR^{F_{t-1}+H_{t-1}}\rightarrow \sR^{F_t}$ be functions, where $F_t\in \sN$ is called the feature dimension of layer $t$. The corresponding \MPNN $\Theta$ is define by the sequence of message functions $(\Phi^{(t)})_{t=1}^T$ and update functions $(\Psi^{(t)})_{t=1}^T$, i.e. 
$\Theta = ((\Phi^{(t)})_{t=1}^T,(\Psi^{(t)})_{t=1}^T)$.
\end{definition}

We now introduce the \onegMPNN with $T$ message-passing layers.
For each node $i \in V$, $\vfonet_i$ at layer $t \in \{1,\ldots,T\}$ is defined recursively using (a) its own representation at layer $t-1$ ($\vfonetmone_i$) and (b) an aggregated representation of its neighbors $m_i^{(t)}$. \CRupdate{Unlike \citet[Definition 2.2]{maskey2022generalization} considering \em{mean aggregation}, we consider here the (N-normalized)sum representation} as follows:%
\CRupdate{
\begin{definition}[\onegMPNN](Adaptation of \cite[Definition 2.2]{maskey2022generalization} to N-normalized GNNs)
\label{def:gmpnn}
Let $(G,\mF)$ be a graph with graph signals as in \Cref{def:rgm} and $\Theta$ be a MPNN as in \Cref{def:MPNN}.
For layer $t=1,...,T$, define %
$\Thetaonet_\mA$ as maps from the input graph $G$ and graph signals $\mF^{(0)}=\mF\in \sR^{N\times F_0}$ to the features in the $t$-th neural layer by
    $$ \Thetaonet_\mA:\sR^{N\times F_0}\rightarrow \sR^{N\times F_t}, \quad \mF \mapsto \mF^{(t)}=(\vfonet_i)_{i=1}^N$$
    where $\mF^{(t)}$ is defined by %
    the ($N$-normalized) sum aggregation procedure, $\forall i \in V$, for $\Thetaonet_\mA$,
    $$m_i^{(t)}:=\frac{1}{N}\sum_{j=1}^N A_{i,j} \Phi^{(t)}(\vfonetmone_i,\vfonetmone_j),$$ and $$
    \vfonet_i:=\Psi^{(t)}(\vfonetmone_i,m_i^{(t)}).
    $$
\end{definition}
}

Given a \onegMPNN, %
$\ThetaoneTcap_\mA$, with $T \geq 1$ layers as in \Cref{def:gmpnn}, their outputs are %
$ \ThetaoneTcap_\mA(\mF) \in \sR^{N  \times F_T}$ for the ($N$-normalized) sum aggregation, and are henceforth denoted as node embedding outputs of the \onegMPNN.  We denote $ \ThetaoneTcap_\mA(\mF)_i$ as the node embedding for node $i \in V$.

\vspace{-5pt}
\subsection{Node embeddings with {\em continuous} message passing neural networks}
\vspace{-5pt}
\update{Here we adapt the degree-normalized definition of \citet{maskey2022generalization} on continuous message passing neural networks for structural node embeddings \CRupdate{ to our continuous integral aggregation (N-normalized GNNs).}} %

\begin{definition}[Continuous message-passing]
\label{def:cmpnn-def}
Given a \MPNN $\Theta$ as in \Cref{def:MPNN}, the {\em node continuous message passing neural network (\onecMPNN)} on graphons and metric-space signals $f:\cX\rightarrow \sR^\firstlayerd$ can be defined by replacing the graph node features and the aggregation scheme in \Cref{def:gmpnn} by the following continuous counterparts. Using a message signal $U:\cX\times \cX\rightarrow \sR^H$, the continuous %
integral aggregation is defined as        
        $\Mone_W(U)(x) = \int_\cX W(x,y) U(x,y)d\mu(y),$
 where $W$ is a graphon. %
\end{definition}
As defined in \citet[Definition 2.4]{maskey2022generalization}, the same \MPNN $\Theta$ can process metric-space signals instead of graph signals with the continuous aggregations. Instead of using continuous mean aggregation as \cite[Definition 2.4]{maskey2022generalization}, we are using continuous integral aggregation.

\CRupdate{
\begin{definition}[\onecMPNN] (Adaptation of \cite[Definition 2.4]{maskey2022generalization} to N-normalized GNNs)
\label{def:cmpnn}
Let $(W,f)$ be a random graph model as in \Cref{def:rgm} and $\Theta$ be a MPNN as in \Cref{def:MPNN}.
For $t=1,...,T$, define $\Thetaonet_W$ as maps from input metric-space signal $\fonezero=f:\cX\rightarrow \sR^{F_0}$ to the features in the $t$-th layer by
    $$\Thetaonet_W:L^2(\cX)\rightarrow L^2(\cX), \quad \fone \mapsto  \fonet,$$
    where $\fonebart$ are defined sequentially
    through the integral aggregation for %
    \begin{align*}
  \Thetaonet_W: \gonet(x):&=\Mone_W(\Phi^{(t)}(\fonetmone,\fonetmone))(x)\\&=\int_\cX W(x,y)\Phi^{(t)}(\fonetmone(x),\fonetmone(y))d\mu(y),
\\
    \fonet(x):&=\Psi^{(t)}(\fonetmone(x),\gonet(x)).
    \end{align*}
\end{definition}
} %
\vspace{-4pt}
\section{Size-stability of node representation and its drawbacks}
\label{sec:stab}
\vspace{-3pt}
We now present our results about convergence of \onegMPNN to \onecMPNN for test graphs $G^\te$ sampled from the graphon random graph model (see \Cref{def:rgm}), and how it leads to size-stability of \onegMPNN for nodes that have the same representation under \onecMPNNs.
In what follows we will focus on the neighbor-average aggregation procedure (a) of \Cref{def:gmpnn}, since this is the more difficult case to prove. 
{\em Similar results for the ($N$-normalized) sum aggregation procedure (\Cref{def:gmpnn}(b)) are shown in the Appendix due to space constraints.}
Moreover, common definitions (e.g., Lipschitz continuous functions) are also relegated to the Appendix to save space.
\vspace{-5pt}
\subsection{Convergence of \gMPNNs towards \cMPNNs as test graph size increase}
\vspace{-5pt}
We now prove that, with high probability, the maximum infinity difference between the \onegMPNN and \onecMPNN node representations decreases with $N^\te$, the size of $G^\te$. 
\update{The proof of \Cref{thm:MainInProb-new} closely follows the \CRupdate{pointwise convergence proof} in  \citet{maskey2022stability}, adapted to our OOD setting and can be found in the Appendix.}%

\CRupdate{
\begin{restatable}[OOD convergence without in-distribution convergence]{theorem}{thmone}
\label{thm:MainInProb-new}
\update{For a random graph model $(W,f)$ satisfying \Cref{def:rgm}, let $N^\tr$ be a random variable defining the distribution of graph sizes in training. Define the test distribution $(G^\te,\mF^\te) \sim (W,f)$ through the causal graph in \Cref{fig:DAG} as an interventional change to obtain larger test graph sizes where \CRupdate{ $\min(\supp(N^\te)) \CRupdate{\gg} M_\tr = \max(\supp(N^\tr))$} (which means any test graph is much larger than the largest possible training graph).}
Let $\Theta = ((\Phi^{(l)})_{l=1}^T, (\Psi^{(l)})_{l=1}^T)$ be a MPNN as in \Cref{def:MPNN} with $T$ layers such that  $\Phi^{(l)}: \mathbb{R}^{2F_{l-1}} \to \mathbb{R}^{H_{l-1}}$ 
and  $\Psi^{(l)}:\mathbb{R}^{F_{l-1} + H_{l-1}} \to \mathbb{R}^{F_l}$ are learned from the training distribution and
are Lipschitz continuous with Lipschitz constants $\LipPhi^{(l)}(M_\tr)$ and $\LipPsi^{(l)}(M_\tr)$ that depend on $M_\tr$. Let \onegMPNN $\ThetaoneTcap_A$ and \onecMPNN $\ThetaoneTcap_W$ be as in \Cref{def:gmpnn,def:cmpnn}. 
Let $X^\te_1,...,X^\te_{N^\te}$ and $\mA^\te$ be as in \Cref{def:rgm}. %
Let $p\in(0, \frac{1}{ \sum_{l=1}^T2(H_l+1)})$.
Then, if \begin{equation}
\begin{aligned}
\label{eq:CondOnN-new}
\frac{\sqrt{N^\te}}{\sqrt{\log{(2N^\te/p)}}} \geq  \frac{4\sqrt{2}}{d_\text{min}} , 
\end{aligned}
\end{equation}
we have with probability at least $1- \sum_{l=1}^T 2(H_l + 1) p$,
\[
\begin{aligned}
\label{eq:deltaAWone}
\deltaAWone := \max_{i=1,...,N^\te} \|\ThetaoneTcap_{\mA^\te}(\mF^\te)_i-\ThetaoneTcap_W(f)(X^\te_i)\|_\infty  \leq (C_1+C_2\|f\|_\infty)\frac{\sqrt{\log (2N^\te/p)}}{\sqrt{N^\te}},
\end{aligned}
\]
where the constants $C_1$ and $C_2$ are defined in the Appendix and depend on $\{\LipPhi^{(l)}(M_\tr),\LipPsi^{(l)}(M_\tr)\}_{l=1}^T$ and the distribution of $(G^\tr,\mF^\tr)$.
\end{restatable}
}
\Cref{thm:MainInProb-new} above shows that as the test graph size $N^\te$ grows, the node representations from the discrete \onegMPNNs \update{learned in the training data} converge to the continuous \onecMPNNs.
{\em \update{\Cref{thm:MainInProb-new}'s OOD statement has profound consequences when it comes to predicting links using the node representations obtained by a \onegMPNN.}}
Next, \Cref{cor:stab-node} shows that for any two nodes $i,j \in V^\te$ that are indistinguishable in the \onecMPNN (defined as $\ThetaoneTcap_W(f)(X^\te_i)=\ThetaoneTcap_W(f)(X^\te_j)$), they will get increasingly similar representations in the discrete \onegMPNN as $N^\te$ grows.

\begin{restatable}{corollary}{corone}
\label{cor:stab-node}
Let $\Theta = ((\Phi^{(l)})_{l=1}^T, (\Psi^{(l)})_{l=1}^T), \ThetaoneTcap_A, \ThetaoneTcap_W, p,(W,f),(G^\tr,\mF^\tr),(G^\te,\mF^\te),N^\tr$, $N^\te, A^\te$, and $X^\te_1,...,X^\te_{N^\te}$ be as in \Cref{thm:MainInProb-new}.
If there exists 
$ i,j\in V^\te, i\neq j,\, \text{s.t. } \ThetaoneTcap_W(X_i)=\ThetaoneTcap_W(X_j)$ and \Cref{eq:CondOnN-new} is satisfied,
then, with $C_1$ and $C_2$ as in \Cref{thm:MainInProb-new}, we have that with probability at least $1- \sum_{l=1}^T 2(H_l + 1) p$,
\[
\begin{aligned} \|\ThetaoneTcap_{\mA^\te}(\mF^\te)_i-\ThetaoneTcap_{\mA^\te}(\mF^\te)_j\|_\infty \leq (C_1+C_2\|f\|_\infty)\frac{2\sqrt{\log (2N^\te/p)}}{\sqrt{N^\te}}.
\end{aligned}
 \]
\end{restatable}

\paragraph{Implications of \Cref{cor:stab-node} on Stochastic Block Models (SBMs).}
\label{sec:SBMs}
In what follows, we will discuss circumstances where two nodes $i,j \in V$ get the same \onecMPNN representations (i.e., both $\ThetaoneTcap_W(f)(X_i)=\ThetaoneTcap_W(f)(X_j)$. %
In what follows we restrict our results to an important family of graphon models: Stochastic Block Models (SBMs)~\citep{snijders1997estimation}, where we also model node attributes. 
SBMs were chosen because they can consistently model large graphs generated by any piecewise Lipschitz graphon model~\citep{airoldi2013stochastic}.
SBMs are also intuitive models, which makes them useful to illustrate our results.

\begin{definition}[Stochastic Block Model (\SBM)]
\label{def:SBM}
An \SBM $(W,f)$ is a random graph model (\Cref{def:rgm}) with cluster structures in $W$ and $f$. Partition the node set into $r \geq 2$ disjoint subsets $S_1,S_2,...,S_r \subseteq V$ (known as blocks or communities) with an associated $r\times r$ symmetric matrix $\mS$, where the probability of an edge $(i,j)$, $i\in \update{S_a}$ and $j\in S_b$ is $\mS_{ab}$, for $a,b\in \{1,\ldots,r\}$. Let $\cX=[0,1]$, and $\mu$ be the uniform distribution on $[0,1]$. By dividing $\cX=[0,1]$ into disjoint convex sets $[t_0,t_1),  [t_1,t_2),\ldots,[t_{r-1},t_r]$, where $t_0=0$ and $t_r=1$, node $i$ belongs to block $S_a$ if $X_i \sim \text{Uniform}(0,1)$ satisfies $X_i \in [t_{a-1},t_a)$. The graphon function $W$ is defined as $W(X_i,X_j)=\sum_{a,b\in \{1,\ldots,r\}}\mS_{ab}\mathds{1}(X_i \in [t_{a-1},t_a))\mathds{1}(X_j \in [t_{b-1},t_b))$. 
We take the liberty to also define node signals in our \SBM model, where for $\mB=[B_1,...,B_r]^T\in \sR^{r\times \firstlayerd}$
the metric-space signal $f:\mathcal{X}\rightarrow \sR^\firstlayerd$ is defined as $f(x)=\sum_{a\in \{1,...,r\}}\mathds{1}(x \in [t_{a-1},t_a))B_a$.
\end{definition}

We define the action of permutation $\pi$ on $\mB$ of \Cref{def:SBM} as $\pi \circ \mB$, where $(\pi \circ \mB)_{\pi_a}=\mB_a$.

\begin{definition}[Isomorphic \SBM blocks] 
\label{def:iso-sbm}
For the \SBM model $(W,f)$ in \Cref{def:SBM}, we say two blocks $a,b \in \{1,\ldots,r\}$ are isomorphic if the \SBM satisfies the following two conditions: 
(a) $t_a-t_{a-1} = t_b-t_{b-1}$, and (b) \update{for $\pi \in \sS_r$, such that $\pi_a=b$, $\pi_b=a$ and $\pi_c=c, \forall c\in \{1,...,r\},\ c\neq a,b$, $\mS=\pi \circ \mS$, and $\mB=\pi \circ \mB$}.
\end{definition}
A similar definition can be obtained for the general graphons in \Cref{def:rgm} using the isomorphic graphon definition of \citet{lovasz2015automorphism}.

Now that we have the definition for isomorphic blocks in \SBM models, we can prove that all nodes in these isomorphic blocks will obtain the same representations under %
integral aggregation \onecMPNNs.

\begin{restatable}{lemma}{lemone}
\label{lem:cmpnn-sbm}
Let $\Theta = ((\Phi^{(l)})_{l=1}^T, (\Psi^{(l)})_{l=1}^T)$ be a MPNN as in \Cref{def:MPNN}, and $\ThetaoneTcap_W$ as in \Cref{def:cmpnn}. For the \SBM model $(W,f)$ in \Cref{def:SBM} with $N^\te$ nodes $X_1, \ldots, X_{N^\te}$. If there exists $i,j\in V^\te$ such that $X^\te_i,X^\te_j$ are nodes that belong to isomorphic \SBM blocks (\Cref{def:iso-sbm}), then $\ThetaoneTcap_W(f)(X^\te_i)=\ThetaoneTcap_W(f)(X^\te_j)$.%
\end{restatable}
Note that even though any two nodes in isomorphic \SBM blocks get the same \onecMPNN representations per \Cref{lem:cmpnn-sbm}, these nodes are likely not isomorphic in $G^\te$ (as shown in \Cref{prop:noniso} in Appendix) and, hence, they get different \onegMPNN representations.
However, \Cref{cor:stab-node} shows that these representations become increasingly similar as the test graph size grows $(N^\te \gg 1)$. We use this observation to understand the ability of \onegMPNNs to perform link prediction tasks next.

\subsection{The hardness of OOD inductive link prediction using structural node embeddings}

The convergence of \onegMPNNs to \onecMPNNs as the test graph size $N^\te$ grows (\Cref{thm:MainInProb-new}) implies through \Cref{cor:stab-node} and \Cref{lem:cmpnn-sbm} that node representations of distinct \SBM blocks can become increasingly similar as the test graph size grows $(N^\te \gg 1)$, even though these nodes are not isomorphic in $G^\te$ with high probability (see \Cref{prop:noniso} in the Appendix).

\begin{definition}[Link prediction function from structural node embeddings]
\label{def:linkpred}
An inductive link prediction function $\etaone:\sR^{F_T} \times \sR^{F_T} \to [0,1]$ takes the \onegMPNN node representations of two nodes $i,j \in V^\te$ and predicts the edge probability $P(\mA^\te_{ij}=1)$.
We assume $\etaone$ is Lipschitz continuous with Lipschitz constant $L_{\etaone}(M_\tr)$ that depends on $\max(\supp(N^\tr))$. In the context of graphon random graph models (\Cref{def:rgm}), we aim to learn  $\etaone(\ThetaoneTcap_{\mA^\te}(\mF^\te)_i,\ThetaoneTcap_{\mA^\te}(\mF^\te)_j) \approx W(i,j)$.
We further assume we predict a link if $\etaone(\cdot,\cdot)>\tau$, while no link if $\etaone(\cdot,\cdot) < \tau$, for some (arbitrary) threshold $\tau\in [0,1]$ chosen by the user of such system.
\end{definition}
The next corollary showcases the difficulty in OOD predicting links using structural node representations as $N^\te$ grows.

\begin{restatable}{corollary}{cortwo}
\label{cor:perf}
Let $\Theta = ((\Phi^{(l)})_{l=1}^T, (\Psi^{(l)})_{l=1}^T)$ be the MPNN with $T$ layers and $\ThetaoneTcap_A, \ThetaoneTcap_W$ as in \Cref{thm:MainInProb-new}. Let $\etaone:\mathbb{R}^{F_T}\times \mathbb{R}^{F_T}\to [0,1]$ be as in \Cref{def:linkpred}. Consider the \SBM $(W,f)$ in \Cref{def:SBM} with isomorphic blocks (\Cref{def:iso-sbm}). 
Let $(G^\tr, \mF^\tr)\sim (W,f)$ and $(G^\te, \mF^\te) \sim (W, f)$ be the training and test graphs with $N^\tr$ and $N^\te$ nodes, respectively. 
Consider any two test nodes
$i,j\in \{1,...,N^\te\}$, $i\neq j$, for which we can make a link prediction decision with $\etaone$ (i.e.,   $\etaone(\ThetaoneTcap_{\mA^\te}(\mF^\te)_i,\ThetaoneTcap_{\mA}(\mF^\te)_j)\neq \tau$).
Let $G^\te$ be large enough to satisfy both \Cref{eq:CondOnN-new} and
\[
    \frac{\sqrt{N^\te}}{\sqrt{\log(2N^\te/p)}}> \frac{2(C_1+C_2\|f\|_\infty)}{|\etaone(\ThetaoneTcap_{\mA^\te}(\mF^\te)_i,\ThetaoneTcap_{\mA^\te}(\mF^\te)_j)-\tau|/L_\etaone(M_\tr)},%
\]
where $p$, $C_1$, and $C_2$ are as given in \Cref{cor:stab-node}.
Then, if $i$ and $j$ belong to isomorphic blocks (i.e., $\ThetaoneTcap_W(f)(X^\te_i)=\ThetaoneTcap_W(f)(X^\te_j)$), with probability at least $1-\sum_{l=1}^T 2(H_l + 1) p$ the link prediction method in \Cref{def:linkpred} will {\em make the same link prediction regardless of the \SBM probability matrix $\mS$ (\Cref{def:SBM}) and whether $i$ and $j$ are in the same block or distinct isomorphic blocks}.
\end{restatable}
\Cref{cor:perf} proves that link prediction with structural node embeddings form \onegMPNNs is  unreliable.
That is, for any link prediction method satisfying \Cref{def:linkpred}, as the test graph grows $N^\te \gg 1$, the method will increasingly struggle to give different predictions within and across isomorphic \SBM blocks, even when these probabilities are arbitrarily different in the underlying graph model. In what follows we show that pairwise embeddings can address this challenge.

\section{Size-stability of structural {\em pairwise} embeddings and its advantages}

We have discussed the limitation of \onegMPNNs on node representation for link prediction. %
Now we claim that a joint continuous message passing graph neural network is capable of link prediction in graphon random graph models (\Cref{def:rgm}).
We define the joint continuous message passing graph neural network inspired by the \onecMPNNs for node representations (\Cref{def:cmpnn}). First, we need to define the {\em graphon fraction of common neighbors} for graphon nodes $x$ and $y$,
$
   \cdW(x,y) := \int_{\mathcal{X}} W(x,z)W(y,z)d\mu(z).
$
We only consider graphons $W$ such that there exists $d_{cmin}$ satisfying $\cdW(x,y)\geq d_\text{cmin}>0, \forall x,y\in \mathcal{X}$ in this section. Since we do not have edge feature as in \Cref{def:rgm}, we define the metric-space pair-wise signal as $f^\pairwisenodes(x,y)=1, \forall x,y\in \mathcal{X}$.

\begin{definition}[\twocMPNN]
\label{def:cmpnn-joint}
Let $(W,f)$ be a random graph model as in \Cref{def:rgm} and $\Theta$ be a MPNN as in \Cref{def:MPNN}.
For $t=1,...,T$, define the continuous (pairwise) \twocMPNN $\Thetatwot_W$ as the mapping that maps input pairwise metric-space signals ${f^\pairwisenodes}^{(0)}=f^\pairwisenodes$ to the features in the $t$-th layer by
    \[
    \Thetatwot_W:L^2(\mathcal{X},\mathcal{X})\rightarrow L^2(\mathcal{X},\mathcal{X}), \quad {f^\pairwisenodes}^{(0)} \mapsto  {f^\pairwisenodes}^{(t)},
    \]
    where ${f^\pairwisenodes}^{(t)}$ are defined recursively by
    \begin{align*}
    {g^\pairwisenodes}^{(t)}(x,y)&:=M^\pairwisenodes_W(\Phi^{(t)}({f^\pairwisenodes}^{(t-1)}))(x,y)=\frac{1}{2} \int_\mathcal{X} (\frac{W(y,z)}{\cdW(x,y)}\Phi^{(t)}({f^\pairwisenodes}^{(t-1)}(x,y),{f^\pairwisenodes}^{(t-1)}(x,z)) \\
    &\qquad\qquad\qquad\qquad\qquad\qquad + \frac{W(x,z)}{\cdW(x,y)}\Phi^{(t)}({f^\pairwisenodes}^{(t-1)}(x,y),{f^\pairwisenodes}^{(t-1)}(y,z))) d\mu(z),\\
    {f^\pairwisenodes}^{(t)}(x,y)&:=\Psi^{(t)}({f^\pairwisenodes}^{(t-1)}(x,y),{g^\pairwisenodes}^{(t)}(x,y)).
    \end{align*}
\end{definition}

The intuition of the aggregation function is that two edges with one same node is considered neighbors in a higher-order graph~\citep{morris2019weisfeiler}, and to go from $(x,y)$ to $(x,z)$, we need to transition from $y$ to $z$, which has probability $W(y,z)$. The same holds for going from $(x,y)$ to $(y,z)$.

\begin{restatable}{lemma}{lemtwo}
\label{lem:stationary}
If $\Phi(x,y)=y$ and $\Psi(x,y)= x/y$, then ${f^\pairwisenodes}^{(t)}(x,y)=W(x,y), \ \forall x,y\in\mathcal{X}$ is a stationary point in the \twocMPNN, i.e. if ${f^\pairwisenodes}^{(t-1)}(x,y)=W(x,y)$, then ${f^\pairwisenodes}^{(t)}(x,y)=W(x,y), \ \forall x,y\in\mathcal{X}$.
\end{restatable}

We define the corresponding \twogMPNN as follows. First we define the \emph{fraction of common neighbors} between nodes $i$ and $j$ as
$ 
    \cdA_{i,j}=\frac{1}{N}\sum_{z=1}^N A_{i,z}\cdot A_{j,z}.
$
If two nodes do not have common neighbors, then we set $\cdA_{i,j}=\frac{1}{N}$ to avoid computation error. Further, we define ${\vfnull^\pairwisenodes}_{i,j}=1$ $\forall i,j\in V$ for any graph $G$, and $\mF^\pairwisenodes=({\vfnull^\pairwisenodes}_{i,j})_{i,j\in V}$ as the pair-wise graph signals.
\begin{definition}[\twogMPNN]
\label{def:gmpnn-joint}Let $(G,\mF)$ be a graph with graph signals as in \Cref{def:rgm} and $\Theta$ be a MPNN as in \Cref{def:MPNN}.
For $t=1,...,T$ layers we define the \twogMPNN $\Thetatwot_A$ as the mapping that maps input pairwise graph signals ${\mF^\pairwisenodes}^{(0)}={\mF^\pairwisenodes}$ to the features in the $t-th$ layer by
    $$\Thetatwot_A:\mathbb{R}^{N^2\times F_0}%
    \rightarrow \mathbb{R}^{N^2\times F_t}%
    , {\mF^\pairwisenodes}^{(0)} \mapsto  {\mF^\pairwisenodes}^{(t)}=({\vfnull^\pairwisenodes}_{i,j}^{(t)})_{i,j=1}^N$$
    where ${\vfnull^\pairwisenodes}^{(t)}$ are defined recursively by the following function,
    \begin{align*}
    {m^\pairwisenodes}_{i,j}^{(t)} &:=\frac{1}{2N}\sum_{z=1}^N \frac{A_{j,z}}{\cdA_{i,j}} \Phi^{(t)}({\vfnull^\pairwisenodes}_{i,j}^{(t-1)},{\vfnull^\pairwisenodes}_{i,z}^{(t-1)})+\frac{A_{i,z}}{\cdA_{i,j}} \Phi^{(t)}({\vfnull^\pairwisenodes}_{i,j}^{(t-1)},{\vfnull^\pairwisenodes}_{j,z}^{(t-1)}),\\
    {\vfnull^\pairwisenodes}_{i,j}^{(t)}&:=\Psi^{(t)}({\vfnull^\pairwisenodes}_{i,j}^{(t-1)},{m^\pairwisenodes}_{i,j}^{(t)}).
    \end{align*}
\end{definition}
\CRupdate{Next, \Cref{thm:MainInProb-new2} shows that these discrete joint representations \twogMPNN converge to the continuous pairwise representation \twocMPNN under the causal DAG of \Cref{fig:DAG}.}
\begin{restatable}[OOD convergence without in-distribution convergence]{theorem}{thmtwo}
\label{thm:MainInProb-new2}
\update{For a random graph model $(W,f)$ satisfying \Cref{def:rgm}, let $N^\tr$ be a random variable defining the distribution of graph sizes in training. Define the test distribution $(G^\te,\mF^\te) \sim (W,f)$ through the causal graph in \Cref{fig:DAG} as an interventional change to obtain larger test graph sizes where \CRupdate{ $\min(\supp(N^\te)) \CRupdate{\gg} M_\tr = \max(\supp(N^\tr))$} (which means any test graph is much larger than the largest possible training graph).}
Let $\Theta = ((\Phi^{(l)})_{l=1}^T, (\Psi^{(l)})_{l=1}^T)$ be a MPNN as in \Cref{def:MPNN} with $T$ layers such that  $\Phi^{(l)}$ 
and  $\Psi^{(l)}$ that are learned from the training data and
are Lipschitz continuous with Lipschitz constants $\LipPhi^{(l)}(M_\tr)$ and $\LipPsi^{(l)}(M_\tr)$. Let \twogMPNN $\ThetatwoTcap_W$ and \twocMPNN $\ThetatwoTcap_W$ be as in \Cref{def:gmpnn-joint,def:cmpnn-joint}. For a random graph model $(W,f)$ as in \Cref{def:rgm} with $d_\text{cmin} > 0$.
Let $X^\te_1,...,X^\te_{N^\te}$ and $\mA^\te$ be as in \Cref{def:rgm}. Let $p\in(0, \frac{1}{ \sum_{l=1}^T2(H_l+1)})$.
Then, if 
$\frac{\sqrt{N^\te}}{\sqrt{\log{(2{(N^\te})^2/p)}}} \geq  \frac{4\sqrt{2}}{d_\text{cmin}}$, 
we have with probability at least $1- \sum_{l=1}^T 2(H_l + 1) p$,
\[
\begin{aligned}
\label{eq:deltaAWtwo}
\deltaAWtwo\! = \!\!\!\!\!\! \max_{i,j=1,...,N^\te} \|\ThetatwoTcap_A({\mF^\pairwisenodes})_{i,j}\!-\!\ThetatwoTcap_W(f^\pairwisenodes)(X^\te_i,X^\te_j)\|_\infty \! \leq \!(C_3\!+\!C_4\|f^\pairwisenodes\|_\infty)\frac{\sqrt{\log (2(N^\te)^2/p)}}{\sqrt{N^\te}},
\end{aligned}
\]
where the constants $C_3$ and $C_4$ are defined in the Appendix and depend on $\{\LipPhi^{(l)}(M_\tr),\LipPsi^{(l)}(M_\tr)\}_{l=1}^T$ and the distribution of $(G^\tr,\mF^\tr)$.
\end{restatable}
Hence, as \update{the test graph size} $N^\te$ gets larger w.r.t.\ $N^\tr$ \update{(that is, as we intervene on the causal DAG of \Cref{fig:DAG} to change the support of the distribution of $N$ in order to obtain larger test graphs)}, the link predictor \update{learned in the training data} using \twogMPNN will converge to a continuous method (\twocMPNN) that can predict links in OOD tasks (i.e., $W(X_i^\te,X_j^\te)$ is a stationary solution of \twocMPNN per \Cref{lem:stationary}). This convergence is also observed in our empirical results.
\vspace{-5pt}
\section{Further Related Work} 
\vspace{-5pt}
In what follows we describe works related to learning transferability in GNNs. 
The concept of transferability of GNN is introduced by \citet{ruiz2020graphon, levie2021transferability}, which state that if two graphs represent same phenomena (e.g., are sampled from the same distribution), then a transferable GNN has approximately the same predictive performance on both graphs. 
This is closely related to in-distribution generalization capabilities of GNNs to unseen test data, i.e., generalization error when train and test data come from the same distribution. Existing works~\citep{keriven2020convergence,ruiz2020graphon,ruiz2021graph,maskey2021transferability} prove the transferability for spectral-based GCNs under graphon models, and \citet{maskey2022generalization} extends these results to more general message passing GNNs. \CRupdate{The GNN smoothness conditions needed to prove uniform convergence of node-embedding equivariant GNNs in \citet{maskey2022generalization} means their GNNs would be unable to perform {\em in-distribution} link prediction in some tasks (such as the graphs in \Cref{def:iso-sbm}). 
However, in practice, we observe (\Cref{sec:expe}) that GNNs are capable of performing these in-distribution link prediction tasks. 
\update{Our results are also based on general message passing GNNs. 
Our goal (OOD link prediction) is, however, significantly different than these prior works, which focus on in-distribution graph and node classification.}
The link prediction challenge for node-embedding equivariant GNNs is either in symmetric graphs (\citet{srinivasan2019equivalence}) or OOD (this work).
\Cref{thm:MainInProb-new} says they vanish as the test graphs grow larger, but \Cref{thm:MainInProb-new2} says that our pairwise equivariant representation is capable of performing these OOD link prediction task}.
Related works relating to the representation power, higher order structural and positional link prediction methods (not already covered in our introduction) can be found in \Cref{appx:related} due to space constraints.

\vspace{-5pt}
\section{Empirical Evaluation}
\label{sec:expe}
\vspace{-5pt}
In what follows we empirically validate our theoretical results in two parts. We implement all our models in Pytorch Geometric~\citep{FeyLenssen2019} and make it available\footnote{\scriptsize \url{https://github.com/yangzez/OOD-Link-Prediction-Generalization-MPNN}}.
Due to space constraints we relegate a detailed description of our experiments to the Appendix. 

\vspace{-2pt}
{\bf Convergence and stability.}
First we will empirically validate \Cref{thm:MainInProb-new,thm:MainInProb-new2} and \Cref{cor:stab-node}. 
\CRupdate{Consider an \SBM (\Cref{def:SBM}) with three blocks ($r=3$) and $\mS_{a,a} = 0.55$,  $a=1,2,3$, $\mS_{1,2}=\mS_{2,1}=0.05$, $\mS_{1,3}=\mS_{3,1}=0.02$. Note that one and three are isomorphic blocks (see \Cref{def:iso-sbm}).
We use a randomly initialized GraphSAGE~\citep{hamilton2017inductive} GNN model for node embedding, and test both the $\Phi$ and $\Psi$ of \Cref{lem:stationary}, and a scenario where $\Psi$ is a randomly-initialized MLP for pairwise embeddings.}

Figures~\ref{fig:conv}(a-c) show log-log plots of the convergence of \gMPNNs to their continuous \cMPNN counterparts as the test graph size $N^\te$ increases.
The empirical approximation errors $\deltaAWone$ (\Cref{thm:MainInProb-new}) (\Cref{fig:conv}(a)) and $\deltaAWtwo$ (\Cref{thm:MainInProb-new2}) are shown as a function of the test graph size $N^\te=2^n$, $n=5,...,13$. The empirical results show agreement with the theory since $\deltaAWone$ and $\deltaAWtwo$ are bounded above by $O(\sqrt{\log N^\te}/\sqrt{N^\te})$, which is approximated by the slope $-1/2$ in a log-log plot. 
Figures~\ref{fig:stab-hist}(d-e) show histograms of the difference between \onegMPNN embeddings of different nodes in $G^\te$. Let $\Delta^\onenode_{i,j} := \ThetaonebarTcap_A(\mF)_i-\ThetaonebarTcap_A(\mF)_j$ for $i,j\in V^\te$, $\Delta^\onenode_{i,j}\in \sR^{F_T}$ and further define ${\Delta^\onenode_\text{iso (resp. non-iso)}}_{i,j} := (\Delta^\onenode_{i,j})_{\argmax_k |(\Delta^\onenode_{i,j})_k|}$, where $k \in \{1,\ldots,F_T\}$ is the dimension of the embedding. We use subscript \texttt{iso} (resp.\ \texttt{non-iso}) when $i,j\in V^\te$ are in isomorphic (resp.\ non-isomorphic) SBM blocks (\Cref{def:iso-sbm}). %
As $N^\te$ increases, \Cref{fig:stab-hist}(d) shows that embeddings between isomorphic blocks converge, validating \Cref{cor:stab-node}, while \Cref{fig:stab-hist}(e) shows that non-isomorphic blocks do not.

\begin{figure}[t!!]
    \centering
    \includegraphics[height=1in,width=1in]{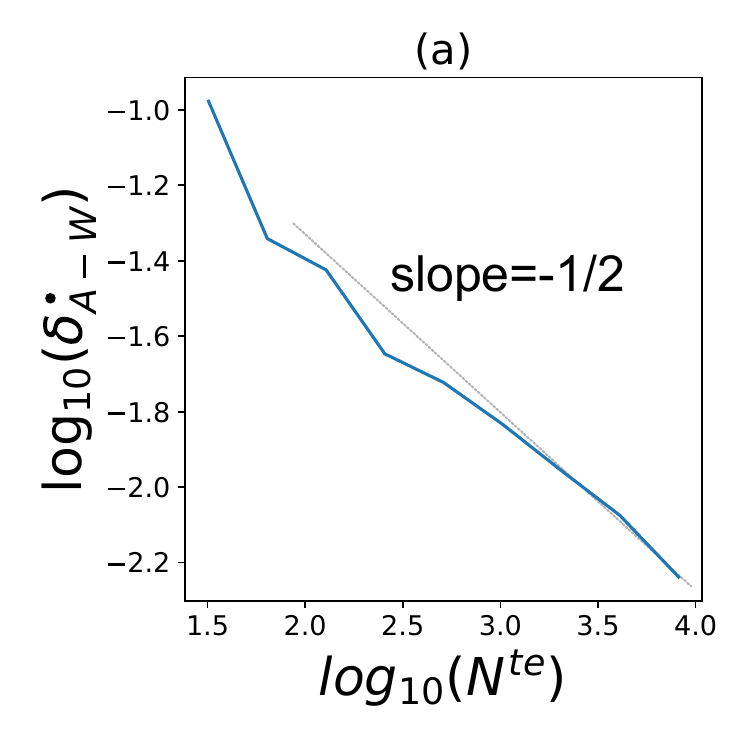}
    \includegraphics[height=1in,width=1in]{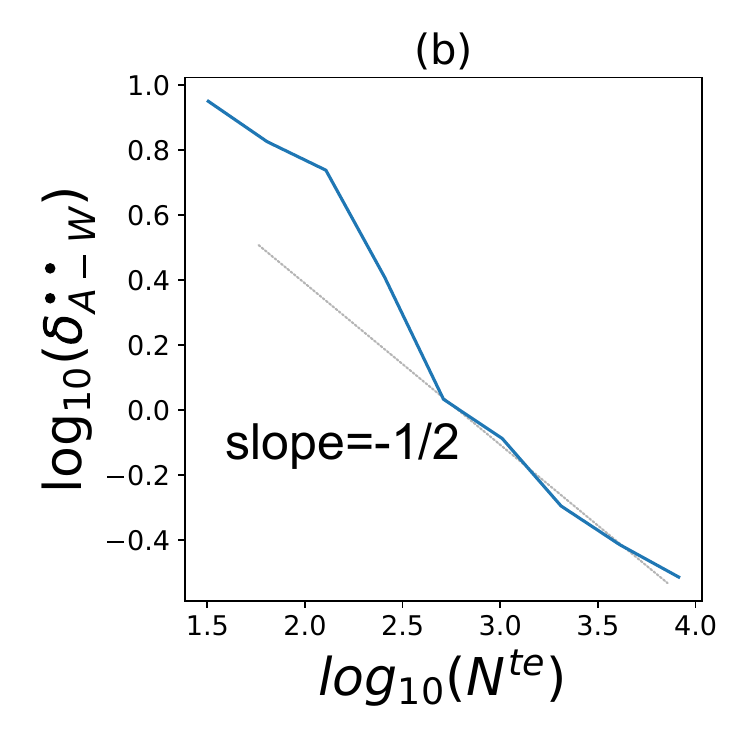}
    \includegraphics[height=1in,width=1in]{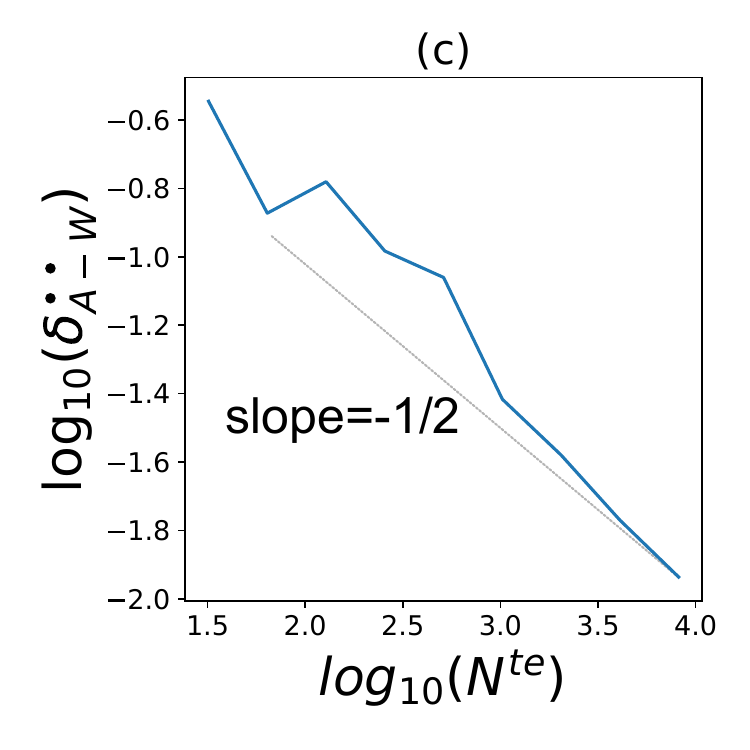}
    \includegraphics[height=1in,width=1in]{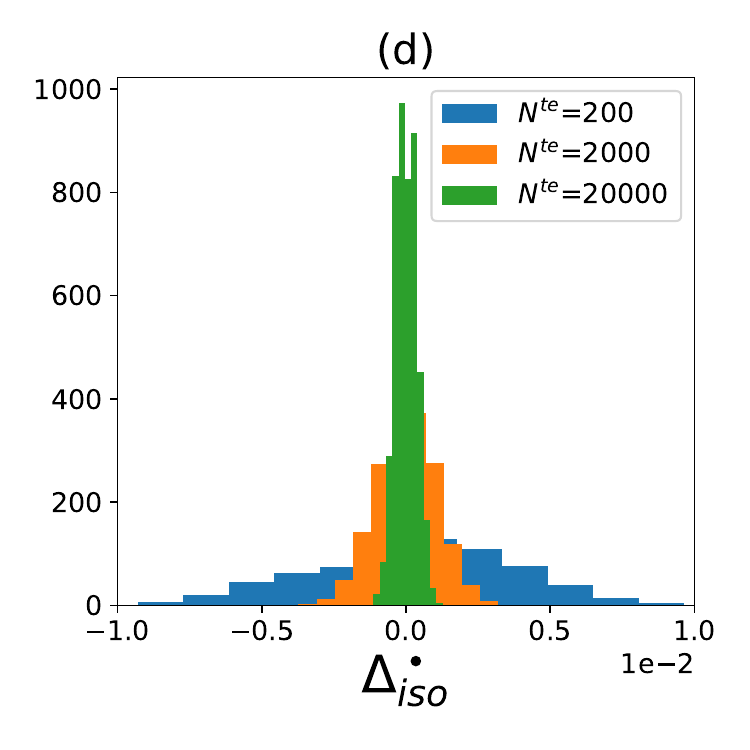}
    \includegraphics[height=1in,width=1in]{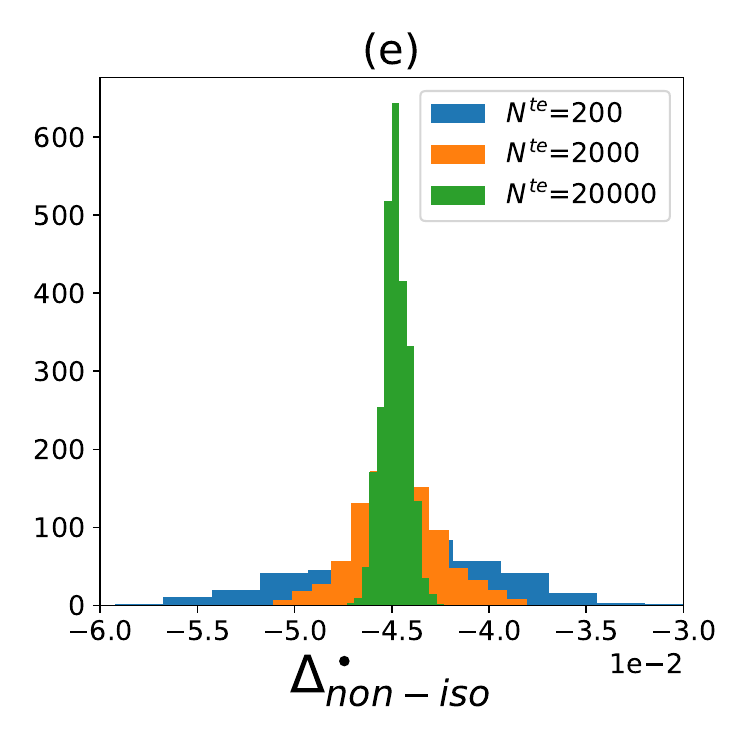}
    \vspace{-5pt}
\caption{{\bf Experimental agreement with theory:} (a) shows $\deltaAWone$ (\Cref{eq:deltaAWone}) of a GraphSAGE GNN as a function of $N^\te$; (b) shows $\deltaAWtwo$ (\Cref{eq:deltaAWtwo}) with the \twogMPNN of \Cref{lem:stationary} as a function of $N^\te$; (c) replicates (b) with $\Psi$ as a randomly-initialized neural network. Results shows close agreement with \Cref{eq:deltaAWone,eq:deltaAWtwo} that predicts slope $\approx -1/2$ in log-log scale for large $N^\te$; (d) shows stable node representations between isomorphic SBM blocks, while (e) shows constant difference in node representations between non-isomorphic SBM blocks, which validate \Cref{cor:stab-node}.}
    \label{fig:conv}
    \label{fig:stab-hist}
\vspace{-20pt}
\end{figure}

{\bf Link prediction performance evaluation with \SBMs (in-distribution and OOD).}
\update{In what follows we introduce empirical results using a \SBM similar in the previous setting. Details can be found in \Cref{appx:linkprediction}.} We start by sampling the training graph $(G^\tr,\mF^\tr)$ with $N^\tr=10^3$ nodes. We randomly hide $10\%$ of $E^\tr$ from the original graph $G^\tr$ for link prediction purpose since the goal of link prediction is to predict possible missing links that is not observed in the original graph. We call these edges $E^\text{hid-tr}$. 
Then we split $E^\text{hid-tr}$ into positive train (80\%) and validation (10\%) edges (we reserve 10\% of $E^\text{hid-tr}$ for the transductive test scenario), and uniformly sample the same number of across-block non-edges as negative train and validation edges. 
The embedding method \onegMPNN (resp.\ \twogMPNN) along with link predictor $\etaone$ (resp.\ $\etatwo$) are trained in an end-to-end manner for predicting positive and negative edges in training using cross-entropy loss.
Our experiments consider three scenarios (in all scenarios we use the same number of negative test edges as positive test edges, sampled from non-edges in $G^\te$ with endpoints in different isomorphic blocks):
(i) (In-distribution) transductive scenario where $G^\te = G^\tr$, where positive test edges are the 10\% reserved in $E^\text{hid-tr}$ not used in training or validation;
(ii) In-distribution inductive scenario where $G^\te$ is sampled from the same \SBM with $N^\te=N^\tr$, where we also hide $10\%$ of the edges and sample $0.1|E^\text{hid-tr}|$ positive test edges from $E^\text{hid-te}$ (for fair comparison across all scenarios); (c) OOD inductive scenario where $G^\te$ is sampled from the same \SBM with $N^\te=10\times N^\tr$, where we also hide $10\%$ of the edges and sample $0.1|E^\text{hid-tr}|$ positive test edges from $E^\text{hid-te}$(for fair comparison across all scenarios).

For {\em structural node embeddings} we consider GraphSAGE~\citep{hamilton2017inductive}, GCN~\citep{Kipf2016} (without positional features), GAT~\citep{velickovic2018graph} and GIN~\citep{xu2018powerful} as the representatives of \onegMPNN models.
The link predictor $\etaone$ is as feedfoward network (with 3 hidden layers and 10 neurons each) that receives the two node embeddings as input, and has link prediction threshold $\tau=0.5$ (see \Cref{def:linkpred} for details). %
 
For {\em structural pairwise embeddings} we choose our proposed \twogMPNN method of \Cref{def:gmpnn-joint}, since we can prove that our approach is theoretically sound in \Cref{lem:stationary}. We test \twogMPNN in two versions: The $\Phi$ and $\Psi$ functions in \Cref{lem:stationary} (denoted {\em fixed $\Psi$}) and a feedforward neural network for $\Psi$ with 2 hidden layers and 5 neurons each (denoted {\em learn $\Psi$}). 
The link predictor $\etatwo$ is the same as $\etaone$ except it just takes one pairwise embedding as input, rather than two node embeddings.

\Cref{table:results} presents our empirical results. The oracle predictor knows the graphon values $W(X_i^\te, X^\te_j)$. Our evaluation metrics include the Matthews correlation coefficient (mcc)~\cite{matthews1975comparison}, balanced accuracy, and Hits@$K$ for $K=10,50,100$ that counts the ratio of positive edges ranked at the $k$-th place or above against all negative edges. Note that \onegMPNN structural node representations can very accurately predict links in the transductive tasks, and still performs reasonably well in inductive in-distribution tasks.
However, as expected from \Cref{cor:perf}, this performance suffers significantly as $N^\te$ becomes $10\times$ larger than $N^\tr$. Now all \onegMPNN methods produce predictors that are no better than a random guess over all metrics (e.g., see OOD mcc and accuracy (in red)). 
In contrast, the \twogMPNN is able to consistently offer good performance on both in-distribution and OOD tasks.

\begin{table*}[t!!!]
\centering
\caption{Test performance over 50 runs of node and pairwise \gMPNNs for in-distribution and OOD link prediction over \SBM graphs. %
Methods marked with $*$ indicate best result out of distinct configurations detailed in the Appendix.
}
\vspace{-5pt}
\label{table:results}
\resizebox{1.\textwidth}{!}{
\begin{tabular}{lllrrrrrrr}
& & &
\multicolumn{5}{c}{Training graph size $\bm{N^\text{tr}=10^3}$}\\
\cmidrule(lr){4-8}
\multicolumn{2}{c}{Tasks} & 
\multicolumn{1}{c}{Model} & 
\multicolumn{1}{c}{Hit@10($\%$)} &
\multicolumn{1}{c}{Hit@50($\%$)} &
\multicolumn{1}{c}{Hit@100($\%$)} &
\multicolumn{1}{c}{mcc.($\%$)} &
\multicolumn{1}{c}{balanced acc.($\%$)} \\
\toprule
\multirow{18}{*}{\rotatebox[origin=c]{90}{In-distribution link prediction}} &
\multirow{9}{*}{\rotatebox[origin=c]{90}{Transductive}} 
   & GraphSAGE* &  95.55(\phantom{0}0.52) & 95.93(\phantom{0}0.73) & 96.14(\phantom{0}0.74) & {\bf 95.42}(\phantom{0}0.37) & {\bf 97.66}(\phantom{0}0.19) \\
&& GCN* & 93.15(14.57) & 93.99(13.08) &  94.35(12.72) & 92.41(14.72) & 95.97(\phantom{0}8.24)  \\
&& GAT* & 93.77(13.03) & 94.01(13.02) & 94.14(13.03) & 90.94(16.09) & 95.26(\phantom{0}8.38)  \\
&& GIN* &   95.77(\phantom{0}0.59) &  96.09(\phantom{0}0.58) & 96.28(\phantom{0}0.59) & 95.48(\phantom{0}0.41) & 97.69(\phantom{0}0.22)   \\
\cmidrule{3-8}
&& {\bf \twogMPNN(fixed $\Psi$)} & 93.76(\phantom{0}0.55) & 94.17(\phantom{0}0.51) & 94.51(\phantom{0}0.49) & 93.64(\phantom{0}0.53) & 96.72(\phantom{0}0.28) \\
&& {\bf \twogMPNN(learn $\Psi$)}  & {\bf 96.71(\phantom{0}0.32)} & {\bf 96.88(\phantom{0}0.31)} & {\bf 97.00(\phantom{0}0.30)} & 94.23(\phantom{0}0.55) & 97.03(\phantom{0}0.29)  \\
\cmidrule{3-8}
&& {\bf Oracle} & 96.92(\phantom{0}0.36) & 96.92(\phantom{0}0.36) & 96.92(\phantom{0}0.36) & 93.74(\phantom{0}0.42) & 96.77(\phantom{0}0.22) \\ 
\cmidrule{2-8}
&\multirow{9}{*}{\rotatebox[origin=c]{90}{Inductive $N^\te=N^\tr$}} 
   & GraphSAGE* & 47.38(39.08) & 52.13(38.87) & 54.94(37.83) & 19.34(43.19) & 61.46(20.17) \\
&& GCN*  & 66.29(37.67) & 68.52(35.87) & 69.92(35.12) & 31.76(35.12) & 67.21(22.75) \\
&& GAT* &  40.05(39.05) & 41.34(39.39) & 41.96(39.54) & 19.44(35.22) & 59.52(16.94)\\
&& GIN* &  39.33(34.62) & 42.93(33.86) & 43.90(33.72) & 18.59(39.43) & 59.79(18.24) \\
\cmidrule{3-8}
&& {\bf \twogMPNN(fixed $\Psi$)} & 93.85(\phantom{0}0.49) & 94.23(\phantom{0}0.51) & 94.55(\phantom{0}0.49) & 93.74(\phantom{0}0.48) & 96.77(\phantom{0}0.25) \\
&& {\bf \twogMPNN(learn $\Psi$)} & {\bf 96.71}(\phantom{0}0.30) & {\bf 96.91}(\phantom{0}0.28) & {\bf 97.02}(\phantom{0}0.27) & {\bf 94.23}(\phantom{0}0.59) & {\bf 97.03}(\phantom{0}0.31)  \\
\cmidrule{3-8}
&&  {\bf Oracle} & 97.01(\phantom{0}0.31) & 97.01(\phantom{0}0.31) & 97.01(\phantom{0}0.31) & 93.87(\phantom{0}0.39) & 96.84(\phantom{0}0.20) \\
\cmidrule{1-8}
\multirow{9}{*}{\rotatebox[origin=c]{90}{{\bf OOD link prediction}}} &
\multirow{9}{*}{\rotatebox[origin=c]{90}{\makecell{\bf Inductive  $N^\te=10^4$}}}
  & GraphSAGE* & \phantom{0}9.97(19.47) & 11.73(21.80) & 12.98(23.70)& {\color{red}-6.56}(\phantom{0}5.12) & {\color{red}49.32}(\phantom{0}0.60)\\
&& GCN*  & 39.29(31.33) & 42.15(30.81) & 44.19(30.97) & {\color{red}-4.88}(14.84) & {\color{red}50.33}(\phantom{0}6.72) \\
&& GAT*  & 27.31(26.93) & 28.13(26.78) & 28.72(26.93) & {\color{red}-2.00}(\phantom{0}8.96) & {\color{red}50.20}(\phantom{0}3.37)\\
&& GIN* & \phantom{0}0.00(\phantom{0}0.00) & \phantom{0}0.00(\phantom{0}0.00) & \phantom{0}0.00(\phantom{0}0.00)& {\color{red}-3.93}(\phantom{0}5.12) & {\color{red}49.59}(\phantom{0}0.57) \\
\cmidrule{3-8}
&& {\bf \twogMPNN(fixed $\Psi$)} &  {\bf 96.74}(\phantom{0}0.07) & {\bf 96.93}(\phantom{0}0.04) & {\bf 97.01}(\phantom{0}0.04) & {\bf 93.76}(\phantom{0}0.05) & {\bf 96.78}(\phantom{0}0.03)\\
&& {\bf \twogMPNN(learn $\Psi$)} & {\bf 96.97}(\phantom{0}0.04) & {\bf 97.02}(\phantom{0}0.04) & {\bf 97.08}(\phantom{0}0.04) & {\bf 93.94}(\phantom{0}0.67) & {\bf 96.88}(\phantom{0}0.35)   \\
\cmidrule{3-8}
&&  {\bf Oracle} & 96.96(\phantom{0}0.03) & 96.96(\phantom{0}0.03) & 96.96(\phantom{0}0.03) & 93.77(\phantom{0}0.04) & 96.79(\phantom{0}0.02) 
\vspace{1pt}
\\
\bottomrule
\end{tabular}
}
\vspace{-15pt}
\end{table*}

\update{
{\bf Link prediction performance evaluation with ogbl-ddi (in-distribution and OOD).}
In what follows we introduce empirical results using the ogbl-ddi dataset, which represents a drug-drug interaction network. For the purpose of performing OOD tasks, we start by sampling $10\%$ of the nodes ($427$ nodes) and its induced subgraph to be the training graph.
Further experimental details can be found in \Cref{appx:ogbl-ddi}. The in-distribution inductive scenario has $G^\te$ constructed as an induced subgraph with $N^\te=N^\tr$ nodes from the remaining ogbl-ddi graph. Our OOD inductive scenario has $G^\te$ as the induced subgraph without the training nodes ($N^\te=3840$ nodes).
The test edges are obtained by applying the original edge split on the newly induced test subgraph, where we further down-sample to the same amount of test edges as in our in-distribution scenarios for fair comparison across all scenarios.
\Cref{table:results-ddi} in the Appendix presents our empirical results on the ogbl-ddi link prediction task. All \onegMPNN methods performs worse in inductive settings than transductive settings, and suffer much worse performance in OOD transductive setting except GCNs.
In contrast, the \twogMPNN is able to consistently offer good performance on both in-distribution and OOD tasks, showing that the theoretical results are not limited to SBM  models.}

\vspace{-5pt}
\section{Conclusions}
\vspace{-5pt}
This work studied and provided the first theoretical framework for the task of out-of-distribution (OOD) link prediction, where test graphs are larger than training graphs.
Using non-asymptotic bounds, this work showed that OOD link prediction methods using structural node embeddings given by message-passing GNNs converge to link predictors that may perform no better than random guesses. \update{The work also proposed a theoretically-sound structural pairwise embedding with a message-passing algorithm which is able to perform our OOD link prediction task by being approximately invariant to interventions on test graph sizes\CRupdate{, as the discrete joint embedding converges to the continuous one. This means that as graph sizes grow in test (OOD), it is still possible to find neural networks parameters that allows our joint representation to converge to the true link probability. We show that the same is not guaranteed for node-embedding equivariant message-passing GNNs.}}
Extensive empirical evaluation showed agreement with these theoretical results. We do not foresee adverse social impacts for this theoretical work\CRupdate{, but it does raise awareness of the shortcomings of node-embedding equivariant massage-passing GNNs for link prediction tasks in applications such as recommender systems}. 

\putbib[relate_paper,link_prediction,invariance,causality,gnns]

\newpage
\section*{Appendix of ``{\em \mytitle{}}''}
\appendix
In \Cref{appx:related}, we introduce more related work that has not been discussed in the main paper. In \Cref{appd:exp}, we provide more details in experiments set up and model training. In \Cref{sec:def-app}, we introduce notations and definitions that we will use throughout the rest of the appendix. In \Cref{appd:large}, we show large random and real world graphs have few isomorphic nodes. In \Cref{appd:thm1}, we prove the convergence results (\Cref{thm:MainInProb-new}) for \onegMPNN when different aggregation functions are used. In \Cref{appd:link-proof}, we prove the results for hardness of link prediction for \onegMPNN. Finally, we prove the convergence results for \twogMPNN and \twocMPNN (\Cref{thm:MainInProb-new2}) in \Cref{appd:pair}.

\section{Further Related Work}
\label{appx:related}
\paragraph{Representation power of GNNs.} The representation power of GNNs is widely studied in recent years. \cite{xu2018powerful,morris2019weisfeiler} first show that gMPNN is no more powerful than 1-WL test \cite{weisfeiler1968reduction}. Many works have been proposed~\citep{morris2019weisfeiler,maron2019provably,pmlr-v97-murphy19a,morris2021weisfeiler} to increase the representation power of GNNs for graph representation, but little has studied on representation power for node and link prediction.

\paragraph{Structural link prediction.} Existing link prediction methods assume that, with powerful enough node representations, combining them can guarantee powerful link representations~\citep{kipf2016variational,grover2016node2vec}. However, \citet{hu2020open} empirically shows that these approaches perform worse than simple heuristic approaches such as Common Neighbor and Adamic-Adar \citep{liben2007link,adamic2003friends}. Theoretically, \citet{srinivasan2019equivalence} was the first work to formally analyze the difference between structural node and link representations, and show that even most-expressive structural node representations are not able to perform link prediction tasks in graphs with high degree of symmetry. In order to remedy this, the state-of-the-art link prediction methods like SEAL~\citep{zhang2018link} use GNNs but transform the task into a graph classification task (the link is an attribute of an induced subgraph around the two target end nodes), where each node in the subgraph are labeled according to their distances to the pair of target end nodes. \citet{zhang2021labeling} unifies such approaches~\citep{zhang2018link,li2020distance,you2021identity} through a method they call ``labeling trick'', which they show is able to learn structural link representations with a node-most expressive GNN.

\paragraph{Ability of GNNs to emulate graph algorithms as graph sizes increase.}
Recently, \citet{xu2020neural} shows that GNNs can extrapolate in algorithmic-related tasks as the graph size grows, if the GNN uses max as an aggregator (rather than mean and sum we considered in this paper).
Unfortunately, our  \Cref{def:gmpnn} of \onegMPNN does not allow max aggregators, in part because it is unclear how one could reach stability using the max aggregator. 
Fortunately, while we could not obtain theoretical results using the max aggregator, we can test it empirically. \Cref{table:results-new} reproduces all our empirical results using the max aggregator 
(on GraphSAGE and GAT, since these are the only GNNs designed for the max aggregator).
Our experiments show that the max aggregator, just like the mean and sum aggregators, shows poor OOD performance as test graph sizes increase. Other works \citet{DBLP:conf/icml/BevilacquaZ021,pmlr-v139-yehudai21a} also talk about graph extrapolation as size grows but focus on graph classification. \update{\citet{wu2022discovering, wu2022handling, DBLP:journals/corr/abs-2202-05441} also explore environment-invarian GNN representations for graph classification or node classification tasks. These works differ in that they focus on node classification and graph classification. As \citet{srinivasan2019equivalence} shows, link prediction tasks are significantly different from graph and node classification.} Moreover, whether or not one can prove that the max aggregator is or is not able to perform our OOD task is left as future work.

\paragraph{Positional node embeddings for link prediction.}
Another way to perform link prediction tasks is to use positional node embeddings (PE), which preserves relative positions of the nodes in a graph. 
The original link predictor in \citet{Kipf2016} uses positional embedding as node attributes for this type of task.
However, such approaches can lose the desired permutation equivariance property in graph models. Traditional PE methods include DeepWalk~\citep{perozzi2014deepwalk} and matrix factorization~\citep{mnih2007probabilistic,ahmed2013distributed}. \citet{you2019position} proposes position-aware GNN that only aggregates message from randomly selected anchor nodes, which has poor generalization ability on inductive tasks. \citet{srinivasan2019equivalence} proves that using set representations of PE embeddings over all permutations of the graph input and all random decisions made by the embedding algorithm (e.g., the set of all eigenvectors of randomly permuted graph Laplacian matrices and random eigenvectors due to geometric multiplicity of eigenvalues) can achieve the desired permutation equivariance for link prediction. \citet{dwivedi2021generalization,kreuzer2021rethinking} propose PE that randomly flips of the sign of eigenvectors to alleviate sign ambiguity and pass it to a transformers architecture~\citep{NIPS2017_3f5ee243}.
\citet{lim2022sign} proposes a representation that is invariant to the elements of the set described by \citet{srinivasan2019equivalence} in order to achieve equivariant representations for spectral node embeddings. 
\citet{wang2022equivariant} proposes a provable solution for using PEs to learn equivariant and stable representation using separate channels in GNN layers. \citet{dwivedi2022graph} turns to the idea of learning PE that can be combined with structural representations, and design architecture to decouple structural and positional representations in order to improve both representations.

\section{Further Experiment details}
\label{appd:exp}
In this section we present the details of the experimental section, discussing implementation details.
Training was performed on NVIDIA Telsa P100, GeForce RTX 2080 Ti, and TITAN V GPUs.

\subsection{Model implementation}

All neural network approaches, including the models proposed in this paper, are implemented in PyTorch~\citep{pytorchcitation} and Pytorch Geometric~\citep{FeyLenssen2019} (available
 respectively under the BSD and MIT license).

Our GIN~\citep{xu2018powerful}, GCN~\citep{Kipf2016}, GAT~\citep{velickovic2018graph} and GraphSAGE~\cite{hamilton2017inductive} implementations are based on their Pytorch Geometric implementations. We also consider max aggregation as proposed by~\citet{xu2021how} for extrapolations although it does not fit our theoretical framework.

We use the Adam optimizer to optimize all the neural network models. We use the neural network weights that achieve best validation-set performance for prediction.

\subsection{Empirical validation for convergence and Stability} \label{appx:convergence}

Consider an \SBM (\Cref{def:SBM}) with three blocks ($r=3$) and $\mS_{a,a} = 0.55$,  $a=1,2,3$, $\mS_{1,2}=\mS_{2,1}=0.05$, $\mS_{1,3}=\mS_{3,1}=0.02$. The probability a node belongs to block one or three is $0.45$, while for block two it is $0.1$.
Note that one and three are isomorphic blocks (see \Cref{def:iso-sbm}).
Since our results are valid for any \gMPNN functions $\Theta$, 
for our first experiment with node embeddings we use a randomly initialized GraphSAGE~\citep{hamilton2017inductive} GNN model, where following standard GNN procedures we initialize node features as size-normalized degrees (where $d_i = \frac{1}{N}\sum_{j=1,...,N}A_{i,j}$).
For the experiment with pairwise embeddings, we test both the $\Phi$ and $\Psi$ of \Cref{lem:stationary}, and a scenario where $\Psi$ is a randomly-initialized feedforward neural network.
Later in this section we show how to efficiently compute the exact \onecMPNN and \twocMPNN embeddings of our GraphSAGE %
and \twogMPNN models.

The validation procedure follows \citet{maskey2022stability}. We use SBM graphs as examples. Consider an \SBM (\Cref{def:SBM}) with three blocks ($r=3$) and $\mS = \begin{bmatrix}
0.55 & 0.05 & 0.02\\
0.05 & 0.55 & 0.05\\
0.02 & 0.05 & 0.55\\
\end{bmatrix}$. The probability a node belongs to block one or three is $0.45$, while for block two it is $0.1$. The in-block edge probability is $0.55$, and across-isomorphic block probability is $0.02$ and across-non-isomorphic block probability is $0.05$.
Note that blocks one and three are isomorphic blocks (see \Cref{def:iso-sbm}).

Since our results is valid for any \gMPNN functions, we use a randomly initialized GraphSAGE~\citep{hamilton2017inductive} GNN model for our first experiments with node embeddings. Following our \Cref{def:SBM}, the initial node embeddings within the same block should be the same, however, following standard GNN procedures we initialize node features as size-normalized degrees. Note that in theory, node within each block has the same expected graphon degree, but this setting is more realistic and shows a stronger results than proposed in our theorems when initial node embeddings also have variance.

To efficiently calculate exact \onecMPNN embeddings, we need to make use of the property of SBMs, i.e. the graphon values within a block is constant. The graphon degree $d_W$ for nodes in block 1, 2 and 3 is $0.2615,0.1$ and $0.2615$. Then we can write the integral $\int_\mathcal{X}\frac{W(x,y)}{d_W(x)}\Phi^{(t)}({\fonebar}^{(t-1)}(x), {\fonebar}^{(t-1)}(y))d\mu(y)$ as $\frac{1}{0.2615}(0.45\times\mS_{1,1}\Phi^{(t)}(\mB^{(t-1)}_1,\mB^{(t-1)}_1)+0.1\times\mS_{1,2}\Phi^{(t)}(\mB^{(t-1)}_1,\mB^{(t-1)}_2)+0.45\times\mS_{1,3}\Phi^{(t)}(\mB^{(t-1)}_1,\mB^{(t-1)}_3))$. This can be calculated exactly by extracting the neural network weights from the GNN model for $\Phi$ and $\Psi$.

Then we compare the difference between \onegMPNN and \onecMPNN for increasing number of nodes. We first plot log-log plots, where a $O(\frac{1}{\sqrt{N}})$ decay rate will have slope $-\frac{1}{2}$ in the log-log plot. Our theory bounds the decay rate by $O(\frac{\log N}{\sqrt{N}})$, which can be approximated by the $-\frac{1}{2}$ slope and is validated in \Cref{fig:conv}.

\paragraph{Pairwise embeddings} For the experiment with pairwise embeddings, we test both the $\Phi$ and $\Psi$ of \Cref{lem:stationary}, and a scenario where $\Psi$ is a randomly initialized two layer feed-forward neural network. To compute the \twocMPNN embeddings, without choosing the adjacency matrix as input to the model, we can input the graphon value matrix $\mW$ where $\mW_{i,j}=W(X_i,X_j)$. In our experiment, we choose graph with $20$ nodes, $9$ in block 1, $2$ in block 2, and $9$ in block 3. The result of \twocMPNN is stable for graphs with different sizes. Then we plot the same log-log plot as above.

\subsection{Link prediction performance evaluation with SBMs}\label{appx:linkprediction}

First, we use a slightly modified \SBM with $\mS = \begin{bmatrix}
0.6 & 0.05 & 0.02\\
0.05 & 0.6 & 0.05\\
0.02 & 0.05 & 0.6\\
\end{bmatrix}$ with other things the same as in the above subsection. Here we increase the in-block edge probability to $0.6$ since we are going to hide edges for link prediction purpose.

\begin{table*}[t!!!]
\centering
\caption{Test performance over 50 runs of node and pairwise \gMPNNs for in-distribution and OOD link prediction over \SBM graphs. %
Methods marked with $*$ indicate best result out of distinct configurations detailed in the Appendix.
}
\vspace{-5pt}
\label{table:results-new}
\resizebox{1.\textwidth}{!}{
\begin{tabular}{lllrrrrrrr}
& & &
\multicolumn{5}{c}{Training graph size $\bm{N^\text{tr}=10^3}$}\\
\cmidrule(lr){4-8}
\multicolumn{2}{c}{Tasks} & 
\multicolumn{1}{c}{Model} & 
\multicolumn{1}{c}{Hit@10($\%$)} &
\multicolumn{1}{c}{Hit@50($\%$)} &
\multicolumn{1}{c}{Hit@100($\%$)} &
\multicolumn{1}{c}{mcc.($\%$)} &
\multicolumn{1}{c}{balanced acc.($\%$)} \\
\toprule
\multirow{18}{*}{\rotatebox[origin=c]{90}{In-distribution link prediction}} &
\multirow{9}{*}{\rotatebox[origin=c]{90}{Transductive}} 
   & GraphSAGE* &  95.55(\phantom{0}0.52) & 95.93(\phantom{0}0.73) & 96.14(\phantom{0}0.74) & {\bf 95.42}(\phantom{0}0.37) & {\bf 97.66}(\phantom{0}0.19) \\
     && GraphSAGE(max)* &  95.43(\phantom{0}0.38) & 96.13(\phantom{0}0.57) & 96.54(\phantom{0}0.60) & {\bf 95.38}(\phantom{0}0.36) & {\bf 97.64}(\phantom{0}0.19) \\
&& GCN* & 93.15(14.57) & 93.99(13.08) &  94.35(12.72) & 92.41(14.72) & 95.97(\phantom{0}8.24)  \\
&& GAT* & 93.77(13.03) & 94.01(13.02) & 94.14(13.03) & 90.94(16.09) & 95.26(\phantom{0}8.38)  \\
&& GAT(max)* & 92.91(12.27) & 93.88(\phantom{0}9.12) & 94.08(\phantom{0}8.82) & 87.36(20.41) & 93.34(10.95)  \\
&& GIN* &   95.77(\phantom{0}0.59) &  96.09(\phantom{0}0.58) & 96.28(\phantom{0}0.59) & 95.48(\phantom{0}0.41) & 97.69(\phantom{0}0.22)   \\
\cmidrule{3-8}
&& {\bf \twogMPNN(fixed $\Psi$)} & 93.76(\phantom{0}0.55) & 94.17(\phantom{0}0.51) & 94.51(\phantom{0}0.49) & 93.64(\phantom{0}0.53) & 96.72(\phantom{0}0.28) \\
&& {\bf \twogMPNN(learn $\Psi$)}  & {\bf 96.71(\phantom{0}0.32)} & {\bf 96.88(\phantom{0}0.31)} & {\bf 97.00(\phantom{0}0.30)} & 94.23(\phantom{0}0.55) & 97.03(\phantom{0}0.29)  \\
\cmidrule{3-8}
&& {\bf Oracle} & 96.92(\phantom{0}0.36) & 96.92(\phantom{0}0.36) & 96.92(\phantom{0}0.36) & 93.74(\phantom{0}0.42) & 96.77(\phantom{0}0.22) \\ 
\cmidrule{2-8}
&\multirow{9}{*}{\rotatebox[origin=c]{90}{Inductive $N^\te=N^\tr$}} 
   & GraphSAGE* & 47.38(39.08) & 52.13(38.87) & 54.94(37.83) & 19.34(43.19) & 61.46(20.17) \\
      && GraphSAGE(max)* &  17.72(22.89) & 25.91(27.75) & 31.43(30.18) & 18.24(30.43) & 58.65(14.53) \\
&& GCN*  & 66.29(37.67) & 68.52(35.87) & 69.92(35.12) & 31.76(35.12) & 67.21(22.75) \\
&& GAT* &  40.05(39.05) & 41.34(39.39) & 41.96(39.54) & 19.44(35.22) & 59.52(16.94)\\
&& GAT(max)* & 41.98(39.23) & 43.34(38.71) & 43.54(38.69) & 22.66(38.99) & 61.46(19.02)  \\
&& GIN* &  39.33(34.62) & 42.93(33.86) & 43.90(33.72) & 18.59(39.43) & 59.79(18.24) \\
\cmidrule{3-8}
&& {\bf \twogMPNN(fixed $\Psi$)} & 93.85(\phantom{0}0.49) & 94.23(\phantom{0}0.51) & 94.55(\phantom{0}0.49) & 93.74(\phantom{0}0.48) & 96.77(\phantom{0}0.25) \\
&& {\bf \twogMPNN(learn $\Psi$)} & {\bf 96.71}(\phantom{0}0.30) & {\bf 96.91}(\phantom{0}0.28) & {\bf 97.02}(\phantom{0}0.27) & {\bf 94.23}(\phantom{0}0.59) & {\bf 97.03}(\phantom{0}0.31)  \\
\cmidrule{3-8}
&&  {\bf Oracle} & 97.01(\phantom{0}0.31) & 97.01(\phantom{0}0.31) & 97.01(\phantom{0}0.31) & 93.87(\phantom{0}0.39) & 96.84(\phantom{0}0.20) \\
\cmidrule{1-8}
\multirow{9}{*}{\rotatebox[origin=c]{90}{{\bf OOD link prediction}}} &
\multirow{9}{*}{\rotatebox[origin=c]{90}{\makecell{\bf Inductive  $N^\te=10^4$}}}
  & GraphSAGE* & \phantom{0}9.97(19.47) & 11.73(21.80) & 12.98(23.70)& {\color{red}-6.56}(\phantom{0}5.12) & {\color{red}49.32}(\phantom{0}0.60)\\
     && GraphSAGE(max)* & \phantom{0}1.44(\phantom{0}2.35) &\phantom{0}2.60(\phantom{0}4.76) & \phantom{0}3.58(\phantom{0}6.53) &{\color{red}-2.52}(\phantom{0}4.44)& {\color{red}49.83}(\phantom{0}0.57) \\
&& GCN*  & 39.29(31.33) & 42.15(30.81) & 44.19(30.97) & {\color{red}-4.88}(14.84) & {\color{red}50.33}(\phantom{0}6.72) \\
&& GAT*  & 27.31(26.93) & 28.13(26.78) & 28.72(26.93) & {\color{red}-2.00}(\phantom{0}8.96) & {\color{red}50.20}(\phantom{0}3.37)\\
&& GAT(max)*  & 32.56(26.94) & 33.01(27.16) & 33.24(27.27) & {\color{red}-2.85}(\phantom{0}9.76) & {\color{red}49.82}(\phantom{0}3.43)\\
&& GIN* & \phantom{0}0.00(\phantom{0}0.00) & \phantom{0}0.00(\phantom{0}0.00) & \phantom{0}0.00(\phantom{0}0.00)& {\color{red}-3.93}(\phantom{0}5.12) & {\color{red}49.59}(\phantom{0}0.57) \\
\cmidrule{3-8}
&& {\bf \twogMPNN(fixed $\Psi$)} &  {\bf 96.74}(\phantom{0}0.07) & {\bf 96.93}(\phantom{0}0.04) & {\bf 97.01}(\phantom{0}0.04) & {\bf 93.76}(\phantom{0}0.05) & {\bf 96.78}(\phantom{0}0.03)\\
&& {\bf \twogMPNN(learn $\Psi$)} & {\bf 96.97}(\phantom{0}0.04) & {\bf 97.02}(\phantom{0}0.04) & {\bf 97.08}(\phantom{0}0.04) & {\bf 93.94}(\phantom{0}0.67) & {\bf 96.88}(\phantom{0}0.35)   \\
\cmidrule{3-8}
&&  {\bf Oracle} & 96.96(\phantom{0}0.03) & 96.96(\phantom{0}0.03) & 96.96(\phantom{0}0.03) & 93.77(\phantom{0}0.04) & 96.79(\phantom{0}0.02) 
\vspace{1pt}
\\
\bottomrule
\end{tabular}
}
\end{table*}

\begin{table*}[t!!!]
\centering
\caption{\update{Test performance over 50 runs of node and pairwise \gMPNNs for in-distribution and OOD link prediction over the ogbl-ddi graph. %
Methods marked with $*$ indicate best result out of distinct configurations detailed in the Appendix.}
}
\vspace{-5pt}
\label{table:results-ddi}
\resizebox{0.9\textwidth}{!}{
\begin{tabular}{lllrrrrrrr}
& & &
\multicolumn{5}{c}{Training graph size $\bm{N^\text{tr}=427}$}\\
\cmidrule(lr){4-8}
\multicolumn{2}{c}{Tasks} & 
\multicolumn{1}{c}{Model} & 
\multicolumn{1}{c}{Hit@10($\%$)} &
\multicolumn{1}{c}{Hit@50($\%$)} &
\multicolumn{1}{c}{Hit@100($\%$)} &
\multicolumn{1}{c}{mcc.($\%$)} &
\multicolumn{1}{c}{balanced acc.($\%$)} \\
\toprule
\multirow{14}{*}{\rotatebox[origin=c]{90}{In-distribution link prediction}} &
\multirow{7}{*}{\rotatebox[origin=c]{90}{Transductive}} 
   & GraphSAGE* &  30.23(2.03) & 47.70(1.75) & 60.36(1.79) &  71.47(0.70) & {\bf 85.72(0.36)} \\
&& GCN* & 17.91(0.52) & 33.69(0.60) & 44.34(0.85) & 59.45(0.50) & 78.85(0.36)  \\
&& GAT* & 1.46(0.52) & 8.20(1.34) & 16.37(1.95) & 52.64(1.62) & 74.75(0.61)  \\
&& GIN* &   17.21(4.74) & 28.76(5.79) & 37.46(6.60) & 54.27(1.59) & 76.84(1.19)  \\
\cmidrule{3-8}
&& {\bf \twogMPNN(fixed $\Psi$)} & 14.09(0.06) & 50.32(0.01) & 65.41(0.01) & \bf{73.23(0.10)} & {\bf 86.60(0.04)} \\
&& {\bf \twogMPNN(learn $\Psi$)}  & {\bf 38.60(1.68)} & {\bf 59.04(0.22)} & {\bf 68.63(0.06)} & 71.96(0.06) & {\bf 85.74(0.03)}  \\
\cmidrule{3-8}
&& {\bf Random} & 0.48(2.58) & 1.16(4.58) & 2.01(6.54) & 0.05(0.39) & 50.00(0.01) \\ 
\cmidrule{2-8}
&\multirow{7}{*}{\rotatebox[origin=c]{90}{Inductive $N^\te=N^\tr$}} 
   & GraphSAGE* & 10.52(1.33) & 23.85(1.29) & 36.60(1.37) & 47.58(2.98) & 71.59(2.46) \\
&& GCN*  & 10.76(0.90) & 24.79(0.73) & 34.99(0.70) & 50.82(0.19) & 74.73(0.21)  \\
&& GAT* & 0.07(0.02) & 0.22(0.10) & 0.51(0.07) & -0.93(0.77) & 50.00(0.01)\\
&& GIN* &  10.95(4.19) & 24.42(5.75) & 33.71(6.70) & 40.67(2.36) & 66.24(1.75) \\
\cmidrule{3-8}
&& {\bf \twogMPNN(fixed $\Psi$)} & 34.24(0.07) & 66.87(0.03) & 73.91(0.02) & {\bf 67.89(0.34)} & {\bf 83.76(0.20)} \\
&& {\bf \twogMPNN(learn $\Psi$)} & {\bf 56.45(0.08)} & {\bf 68.42(0.03)} & {\bf 74.93(0.02)} & 65.55(0.15) &  82.62(0.09)  \\
\cmidrule{3-8}
&& {\bf Random} & 0.41(1.64) & 2.20(4.88) & 4.97(8.77) & -0.03(0.22) & 50.00(0.00)  \\
\cmidrule{1-8}
\multirow{7}{*}{\rotatebox[origin=c]{90}{{\bf OOD link prediction}}} &
\multirow{7}{*}{\rotatebox[origin=c]{90}{\makecell{\bf Inductive  $N^\te=3840$}}}
  & GraphSAGE* & 1.79(1.21) & 13.70(6.71) & 25.31(8.77) & 16.65(3.31) & 52.79(1.01)\\
&& GCN*  & 12.38(1.23) & 27.28(1.27) & 37.45(1.43) & 55.03(0.76) & 77.38(0.36)\\
&& GAT*  & 2.76(1.27) & 7.55(3.28) & 12.78(4.50) & 23.83(16.31) & 59.54(6.96)\\
&& GIN* & 0.00(0.00) & 0.00(0.00) & 0.00(0.00) & 45.87(3.55) & 68.92(2.78) \\
\cmidrule{3-8}
&& {\bf \twogMPNN(fixed $\Psi$)} &  9.31(5.23) & 67.42(0.02) & 78.44(0.01) & {\bf 75.42(0.17)} & {\bf  87.37(0.11)}\\
&& {\bf \twogMPNN(learn $\Psi$)} & {\bf 57.97(0.02)} & {\bf 74.75(0.07)} & {\bf 80.00(0.11)} &  72.04(0.20) &  84.57(0.14)   \\
\cmidrule{3-8}
&& {\bf Random} & 1.21(3.50) & 3.39(7.72) & 5.71(11.13) & 0.00(0.00) & 50.00(0.00)
\\
\bottomrule
\end{tabular}
}
\vspace{-5pt}
\end{table*}

\begin{table*}[t!!!]
\centering
\caption{\update{Test performance over 50 runs of node and pairwise \gMPNNs for in-distribution (large) and OOD (small) link prediction over \SBM graphs. %
Methods marked with $*$ indicate best result out of distinct configurations detailed in the Appendix.}
}
\vspace{-5pt}
\label{table:results-small}
\resizebox{0.9\textwidth}{!}{
\begin{tabular}{lllrrrrrrr}
& & &
\multicolumn{5}{c}{Training graph size $\bm{N^\text{tr}=10^4}$}\\
\cmidrule(lr){4-8}
\multicolumn{2}{c}{Tasks} & 
\multicolumn{1}{c}{Model} & 
\multicolumn{1}{c}{Hit@10($\%$)} &
\multicolumn{1}{c}{Hit@50($\%$)} &
\multicolumn{1}{c}{Hit@100($\%$)} &
\multicolumn{1}{c}{mcc.($\%$)} &
\multicolumn{1}{c}{balanced acc.($\%$)} \\
\toprule
\multirow{14}{*}{\rotatebox[origin=c]{90}{In-distribution link prediction}} &
\multirow{7}{*}{\rotatebox[origin=c]{90}{Transductive}} 
   & GraphSAGE* &  75.35(38.50) & 75.41(38.53) & 75.46(38.55) & 70.81(43.47) & 85.93(20.62)\\
&& GCN* & 86.23(27.88) & 86.48(27.85) & 86.56(27.85) & 82.73(32.32) & 91.36(15.84)  \\
&& GAT* &  59.21(43.07) & 59.62(43.09) & 59.79(43.12) & 50.19(42.84) & 75.51(21.35) \\
&& GIN* & 80.89(33.65) & 81.12(33.71) & 81.20(33.72) & 82.46(30.01) & 90.49(16.59)   \\
\cmidrule{3-8}
&& {\bf \twogMPNN(fixed $\Psi$)} & 95.74(\phantom{0}0.12) & 96.15(\phantom{0}0.06) & 96.33(\phantom{0}0.04) & 93.77(\phantom{0}0.04) & 96.79(\phantom{0}0.02) \\
&& {\bf \twogMPNN(learn $\Psi$)}  & {\bf 96.95(\phantom{0}0.03)} & {\bf 96.95(\phantom{0}0.03)} & {\bf 96.95(\phantom{0}0.03)} & {\bf 93.76(\phantom{0}0.06)} & {\bf 96.79(\phantom{0}0.03)}   \\
\cmidrule{3-8}
&& {\bf Oracle} & 96.96(\phantom{0}0.03) & 96.96(\phantom{0}0.03) & 96.96(\phantom{0}0.03) & 93.77(\phantom{0}0.04) & 96.79(\phantom{0}0.02) \\ 
\cmidrule{2-8}
&\multirow{7}{*}{\rotatebox[origin=c]{90}{Inductive $N^\te=N^\tr$}} 
   & GraphSAGE* & 64.77(40.22) & 65.88(39.91) & 66.60(39.87)  & 33.19(50.16) & 68.30(23.45)\\
&& GCN* & 79.67(34.82) & 79.90(34.55) & 80.07(34.31) & 51.16(49.53) & 76.23(23.72)  \\
&& GAT* &   46.73(37.62) & 47.12(37.64) & 47.31(37.65) & 19.14(39.03) & 60.02(18.77)  \\
&& GIN* &  59.68(41.62) & 61.15(41.47) & 61.69(41.37) & 44.80(46.15) & 71.90(22.57) \\
\cmidrule{3-8}
&& {\bf \twogMPNN(fixed $\Psi$)} & 95.67(\phantom{0}0.11) & 96.15(\phantom{0}0.06) & 96.33(\phantom{0}0.04) & 93.77(\phantom{0}0.05) & 96.79(\phantom{0}0.03) \\
&& {\bf \twogMPNN(learn $\Psi$)} & {\bf 96.94(\phantom{0}0.04)} & {\bf 96.94(\phantom{0}0.04)} & {\bf 96.94(\phantom{0}0.04)} & {\bf 93.76(\phantom{0}0.06)} & {\bf 96.78(\phantom{0}0.03)}  \\
\cmidrule{3-8}
&&  {\bf Oracle} & 96.95(\phantom{0}0.04) & 96.95(\phantom{0}0.04) & 96.95(\phantom{0}0.04) & 93.77(\phantom{0}0.05) & 96.79(\phantom{0}0.03) \\
\cmidrule{1-8}
\multirow{7}{*}{\rotatebox[origin=c]{90}{{\bf OOD link prediction}}} &
\multirow{7}{*}{\rotatebox[origin=c]{90}{\makecell{\bf Inductive  $N^\te=10^3$}}}
  & GraphSAGE* & 33.52(44.93) & 33.70(44.87) & 33.97(44.77) & 32.72(47.00) & 66.97(22.73)\\
&& GCN*  & 72.28(40.06) & 73.95(38.58) & 74.17(38.56) & 68.93(40.98) & 84.54(19.69) \\
&& GAT*  & 23.31(39.07) & 23.32(39.07) & 23.34(39.07) & 24.07(39.18) & 61.74(19.39)\\
&& GIN* & \phantom{0}1.31(\phantom{0}1.62) & \phantom{0}1.39(\phantom{0}1.62) & \phantom{0}1.42(\phantom{0}1.62)& {\color{red}-0.64}(\phantom{0}5.63) & {\color{red}49.93}(\phantom{0}0.63) \\
\cmidrule{3-8}
&& {\bf \twogMPNN(fixed $\Psi$)} &  93.68(\phantom{0}0.40) & 93.72(\phantom{0}0.41) & 93.74(\phantom{0}0.41) & 93.40(\phantom{0}0.42) & 96.59(\phantom{0}0.22)\\
&& {\bf \twogMPNN(learn $\Psi$)}  & {\bf 96.12(\phantom{0}0.28)} &  {\bf 96.44(\phantom{0}0.34)} & {\bf 96.57(\phantom{0}0.36)} & {\bf 94.43(\phantom{0}0.31)} & {\bf 97.14(\phantom{0}0.16)}  \\
\cmidrule{3-8}
&&  {\bf Oracle} & 96.94(\phantom{0}0.30) & 96.94(\phantom{0}0.30) & 96.94(\phantom{0}0.30) & 93.82(\phantom{0}0.40) & 96.81(\phantom{0}0.21) 
\vspace{3pt}
\\
\bottomrule
\end{tabular}
}
\vspace{-10pt}
\end{table*}

We start by sampling the training graph $(G^\tr,\mF^\tr)$ with $N^\tr=10^3$ nodes. We randomly hide $10\%$ of $E^\tr$ from the original graph $G^\tr$ for link prediction purpose since the goal of link prediction is to predict possible missing links that is not observed in the original graph. We call these edges $E^\text{hid-tr}$.

Then we split $E^\text{hid-tr}$ into positive train (80\%) and validation (10\%) edges (we reserve 10\% of $E^\text{hid-tr}$ for the transductive test scenario), and uniformly sample the same number of across-block non-edges as negative train and validation edges. 
The embedding method \onegMPNN (resp.\ \twogMPNN) along with link predictor $\etaone$ (resp.\ $\etatwo$) are trained in an end-to-end manner for predicting positive and negative edges in training using cross-entropy loss.
Our experiments consider three scenarios (in all scenarios we use the same number of negative test edges as positive test edges, sampled from non-edges in $G^\te$ with endpoints in different isomorphic blocks):
(i) (In-distribution) transductive scenario where $G^\te = G^\tr$, where positive test edges are the 10\% reserved in $E^\text{hid-tr}$ not used in training or validation;
(ii) In-distribution inductive scenario where $G^\te$ is sampled from the same \SBM with $N^\te=N^\tr$, where we also hide $10\%$ of the edges and sample $0.1|E^\text{hid-tr}|$ positive test edges from $E^\text{hid-te}$ (for fair comparison across all scenarios); (c) OOD inductive scenario where $G^\te$ is sampled from the same \SBM with $N^\te=10\times N^\tr$, where we also hide $10\%$ of the edges and sample $0.1|E^\text{hid-tr}|$ positive test edges from $E^\text{hid-te}$(for fair comparison across all scenarios).

For {\em structural node embeddings} we consider GraphSAGE~\citep{hamilton2017inductive}, GCN~\citep{Kipf2016} (without positional features), GAT~\citep{velickovic2018graph} and GIN~\citep{xu2018powerful} as the representatives of \onegMPNN models. Here we also add $\max$ aggregation for GAT and GraphSAGE model as proposed by \citet{xu2021how} for extrapolation.
The link predictor $\etaone$ is as feedfoward network that receives the two node embeddings as input, and has link prediction threshold $\tau=0.5$ (see \Cref{def:linkpred} for details). We initialize the node features as the size-normalized degrees.
 
For {\em structural pairwise embeddings} we choose our proposed \twogMPNN method of \Cref{def:gmpnn-joint}, since we can prove that our approach is theoretically sound in \Cref{lem:stationary}. We test \twogMPNN in two versions: The $\Phi$ and $\Psi$ functions in \Cref{lem:stationary} (denoted {\em fixed $\Psi$}) and a feedforward neural network for $\Psi$ (denoted {\em learn $\Psi$}). 
The link predictor $\etatwo$ is the same as $\etaone$ except it just takes one pairwise embedding as input, rather than two node embeddings. We initialize the pairwise features as all $1$'s to contain no additional information about connectivity between the pair of the nodes.

\update{Many existing link prediction methods rely on positional node embeddings and our work focuses on permutation-equivariant MPNN GNNs. These positional node embedding link prediction methods are not equivariant (they are positional node representations) based on matrix and tensor factorization methods. Developing a theory for the effect of positional node representations in OOD link prediction is far from trivial and an entirely new paper that requires a new theory. At this point we do not even know how positional representations could be approximately counterfactually-invariant.}

For all models including \onegMPNN, \twogMPNN, $\etaone$ and $\etatwo$. The number of hidden layers was chosen between \{$2,3$\}, and the number of hidden neurons was chosen between $\{5,10\}$ due to the simple experimental set up. For GAT, we have $2$ attention heads.   
Specifically. We optimized all models using Adam with learning rate chosen from \{$1\!\times\! 10^{-3},5\!\times\! 10^{-4},1\!\times\! 10^{-4}$\}. We also choose $\etaone$ as taking the inner product between pair of nodes as input (as \citet{hu2020open}) and the concatenated node embeddings as input. The hyperparameter search is performed by training all models with $10$ different initialization seeds and selecting the configuration that achieved the highest mean accuracy on the validation data, and we mark the methods with $*$ in \Cref{table:results,table:results-new} indicating the optimal configuration is being used. The training time is around $10$ minutes for $1,000$ epochs.

\Cref{table:results-new} presents our empirical results in the new setting over $50$ independent runs. The oracle predictor knows the graphon values $W(X_i^\te, X^\te_j)$. The reason why it can not achieve $100\%$ accuracy is because there exists rarely sampled positive edges between blocks. Our evaluation metrics include the Matthews correlation coefficient (mcc)~\cite{matthews1975comparison}, balanced accuracy, and Hits@$K$ for $K=10,50,100$ that counts the ratio of positive edges ranked at the $k$-th place or above against all negative edges. The results from the new table conveys the same message as \Cref{table:results} and has been discussed in \Cref{sec:expe}.

\update{We also include a new setting for training on larger graphs ($10^4$ nodes) and extrapolating to smaller graphs ($10^3$ nodes) in \Cref{table:results-small}. We are able to see the structure node representations \onegMPNN are still able to perform relatively well on in-distribution inductive tasks, although the graphs are large, while still suffer from OOD performance to smaller graphs except GCN, although it is not related to the theoretical discussions of this paper. In contrast, the \twogMPNN is able to consistently offer good performance on both in-distribution and OOD tasks.}

As discussed in \Cref{appx:related}, \citet{xu2020neural} shows that GNNs can extrapolate in algorithmic-related tasks as the graph size grows, if the GNN uses max as an aggregator (rather than mean and sum we considered in this paper).
Unfortunately, our  \Cref{def:gmpnn} of \onegMPNN does not allow max aggregators, in part because it is unclear how one could reach stability using the max aggregator. 
Fortunately, while we could not obtain theoretical results using the max aggregator, we can test it empirically. \Cref{table:results-new} reproduces all our empirical results using the max aggregator 
(on GraphSAGE and GAT, since these are the only GNNs designed for the max aggregator).
Our experiments show that the max aggregator, just like the mean and sum aggregators, shows poor OOD performance as test graph sizes increase.
Whether there is theoretical proof that the max aggregator is not able to perform this OOD task is left as future work.

\update{\subsection{Link prediction performance evaluation with ogbl-ddi}\label{appx:ogbl-ddi}

In what follows we introduce empirical results using the ogbl-ddi dataset, which represents a drug-drug interaction network. For the purpose of performing OOD tasks, we start by sampling $10\%$ of the nodes ($427$ nodes) and its induced subgraph to be the training graph, where node features are constructed as size-normalized degrees in the training graph. Validation positive and negative edges are obtained by applying the original edge split on the induced training subgraph.
Our experiments consider three scenarios:
(i) (In-distribution) transductive scenario where $G^\te = G^\tr$, where test positive and negative edges are obtained by applying the original edge split on the induced training subgraph;
(ii) In-distribution inductive scenario where $G^\te$ is constructed as sampling $N^\te=N^\tr$ nodes from the remaining ogbl-ddi graph and its induced subgraph, where the test edges are obtained by applying the original edge split on the newly induced test subgraph; (iii) OOD inductive scenario where $G^\te$ is the induced subgraph without the training nodes with $N^\te=3840$, the test edges are obtained by applying the original edge split on the newly induced test subgraph, where we further down-sample to the same amount of test edges as in (ii) for fair comparison across all scenarios.

We used the same benchmarking methods as in the SBM experiments, and add a random guesser where it is constructed as randomly-initialized GraphSAGE model with a randomly-initialized link predictor. We initialize the pairwise features as all $1$'s to contain no additional information about connectivity between the pair of the nodes for \twogMPNN.

For {\em structural node embeddings} we consider GraphSAGE~\citep{hamilton2017inductive}, GCN~\citep{Kipf2016} (without positional features), GAT~\citep{velickovic2018graph} and GIN~\citep{xu2018powerful} as the representatives of \onegMPNN models.
The link predictor $\etaone$ is as feedfoward network that receives the two node embeddings as input, and has link prediction threshold $\tau=0.5$ (see \Cref{def:linkpred} for details). We initialize the node features as the size-normalized degrees.
 
For {\em structural pairwise embeddings} we choose our proposed \twogMPNN method of \Cref{def:gmpnn-joint}, since we can prove that our approach is theoretically sound in \Cref{lem:stationary}. We test \twogMPNN in two versions: The $\Phi$ and $\Psi$ functions in \Cref{lem:stationary} (denoted {\em fixed $\Psi$}) and a feedforward neural network for $\Psi$ (denoted {\em learn $\Psi$}). 
The link predictor $\etatwo$ is the same as $\etaone$ except it just takes one pairwise embedding as input, rather than two node embeddings. We initialize the pairwise features as all $1$'s to contain no additional information about connectivity between the pair of the nodes.

For all models including \onegMPNN, \twogMPNN, $\etaone$ and $\etatwo$. The number of hidden layers was chosen between \{$2,3$\}, and the number of hidden neurons was chosen between $\{16,32\}$ due to the simple experimental set up. For GAT, we have $2$ attention heads.   
Specifically. We optimized all models using Adam with learning rate chosen from \{$1\!\times\! 10^{-3},5\!\times\! 10^{-4},1\!\times\! 10^{-4}$\}. We also choose $\etaone$ as taking the inner product between pair of nodes as input (as \citet{hu2020open}) and the concatenated node embeddings as input. We train all the models with 200 epochs. The hyperparameter search is performed by training all models with $10$ different initialization seeds and selecting the configuration that achieved the highest mean accuracy on the validation data, and we mark the methods with $*$ in \Cref{table:results-ddi} indicating the optimal configuration is being used.

\Cref{table:results-ddi}  presents our empirical results on the ogbl-ddi link prediction task. All \onegMPNN methods performs worse in inductive settings than transductive settings, and suffer much worse performance in OOD transductive setting except GCNs.
In contrast, the \twogMPNN is able to consistently offer good performance on both in-distribution and OOD tasks, showing that the theoretical results are not limited to SBM  models.
} %

\section{Defintion and notations}
\label{sec:def-app}
\update{In this section, we follow the definitions and notations from \citet[Appendix A]{maskey2022generalization}.}
As in \citet[Appendix A]{maskey2022generalization}, we call the metric space $(\chi,d)$, where the metric in the space $\chi$ is defined as $d:\chi\times\chi\rightarrow \left[0,\infty\right)$.
The nodes of the graph are considered as sampled point from $\chi$, the node $i$ is identified with $X_i$ for the graph $G$ with nodes $X=(X_1,\ldots,X_N)$. We also represent $\mF(X_i) := \vfnull_i$ for $i= 1,\ldots, N$.

Next, we define various notions of degree for the pairwise node embedding.

\CRupdate{
\begin{definition}
\label{def:allgraphon} Let $W$ defined in \Cref{def:rgm}, and $G$ as the sampled graph with nodes $X=(X_1,...,X_N)$.
\begin{itemize}
    \item We define the graphon fraction of common neighbors at $x,y\in \mathcal{X}$ by
    \begin{equation}
        c_W(x,y) = \int_\mathcal{X} W(x,z)W(y,z)d\mu(z),
    \end{equation}
    \item Given two points $x,y$ that need not be in $X$, we define the graph-graphon fraction of common neighbors of $X$ at $x,y$ by
    \begin{equation}
        c_X(x,y) = \frac{1}{N}\sum_{i=1}^N W(x,X_i)W(y,X_i),
    \end{equation}
    \item Given two points $x,y$ that need not be in $X$, we define the sampled-graph fraction of common neighbors of $X$ at $x,y$ by
    \begin{equation}
        c_A(x,y) = \frac{1}{N}\sum_{i=1}^N A(x,X_i)A(y,X_i),
    \end{equation}
    where we define $A(x,X_i)\sim \textit{Ber}(W(x,X_i))$ and $A(y,X_i)\sim \textit{Ber}(W(y,X_i))$ as independent random variables.
\end{itemize}
\end{definition}
}
where $c_X(x,y)$ and $c_A(x,y)$ are interpreted as the graph fraction of common of neighbors of the node pair $(x,y)$ in the graph $(x,y,X_1,...,X_n)$.

\CRupdate{Adapting \citet[Definition A.3]{maskey2022generalization} to the continuous integral aggregation,}
\begin{definition} 
\label{def:graphon-degree-avg}
Let $W$ be defined in \Cref{def:rgm}, for a metric-space message signal $U:\mathcal{X}\times \mathcal{X}\rightarrow \mathbb{R}^F$, the continuous integral aggregation is defined by
$$\Mone_WU=\int_\mathcal{X}W(\cdot,y)U(\cdot,y)d\mu(y).$$
\end{definition}

\CRupdate{Adapting \citet[Definition A.4]{maskey2022generalization} to the N-normalized sum aggregation,}
\begin{definition}
\label{def:graphon-neighbor-avg}
Let $W$ be defined in \Cref{def:rgm}, $X=X_1,...,X_n$ sample points. For a metric-space message signal $U:\mathcal{X}\times \mathcal{X}\rightarrow \mathbb{R}^F$, we define the graph-graphon (N-normalized) sum aggregation by
$$\Mone_XU = \frac{1}{N}\sum_iW(\cdot,X_i)U(\cdot,X_i),$$
and the sampled-graph (N-normalized) sum aggregation by
$$\Mone_AU = \frac{1}{N}\sum_iA(\cdot,X_i)U(\cdot,X_i),$$
where we define $A(x,X_i)\sim \textit{Ber}(W(x,X_i))$ as a random variable.
\end{definition}

\begin{definition}
Let $W$ be defined in \Cref{def:rgm}, $X=X_1,...,X_n$ sample points. For a metric-space message signal $U:(\mathcal{X}\times \mathcal{X})\times(\mathcal{X}\times \mathcal{X}) \rightarrow \mathbb{R}^F$, we define the graphon pairwise aggregation by
$$M^\pairwisenodes_W U = \frac{1}{2}\int_{\mathcal{X}} (\frac{W(y,z)}{c_W(\cdot,\cdot)}U(\cdot, (x,z))+\frac{W(x,z)}{c_W(\cdot,\cdot)}U(\cdot, (y,z)))d\mu(z),$$
and the graph-graphon pairwise aggregation by
$$M^\pairwisenodes_X U = \frac{1}{2N}\sum_{i=1}^N (\frac{W(y,X_i)}{c_X(\cdot,\cdot)}U(\cdot, (x,X_i))+\frac{W(x,X_i)}{c_X(\cdot,\cdot)}U(\cdot, (y,X_i))),$$
and the sample-graph pairwise aggregation by
$$M^\pairwisenodes_A U = \frac{1}{2N}\sum_{i=1}^N (\frac{A(y,X_i)}{c_A(\cdot,\cdot)}U(\cdot, (x,X_i))+\frac{A(x,X_i)}{c_A(\cdot,\cdot)}U(\cdot, (y,X_i))),$$
where we define $A(x,X_i)\sim \textit{Ber}(W(x,X_i))$ as a random variable.
\end{definition}

\citet[Definition A.7]{maskey2022generalization} has defined for a vector $\mathbf{z}=(z_1,\ldots,z_F) \in \mathbb{R}^F$, we define as usual
\[
\|\mathbf{z}\|_\infty = \max_{ 1 \leq k \leq F } |z_k|.
\]
\CRupdate{
For every $x,x'\in \mathcal{X}$, we say a function $f:\chi \to \mathbb{R}^F$ is Lipschitz continuous if there exists a $L_f > 0$ such that for every  $x,x' \in \chi$, we have
\[
\|f(x) - f(x')\|_\infty \leq L_f d(x,x').
\]
Here if $\chi = \mathbb{R}^F$, $d(x,x') = \| x- x'\|_\infty.$
For our theoretical results, we make the following assumptions:}

\begin{assumption}(extension of \citet[Definition A.10]{maskey2022generalization})
\label{ass:graphon}
Let $(\chi,d)$ be a metric space and
 $W: \chi \times \chi \rightarrow [0, \infty)$. 
 Let $\Theta$ be a MPNN with message and update  functions $\Phi^{(l)}: \mathbb{R}^{2F_l} \rightarrow \mathbb{R}^{H_l}$ and  $\Psi^{(l)}: \mathbb{R}^{F_l+H_l} \rightarrow \mathbb{R}^{F_{l+1}}$, $l=1,\ldots,T-1$. 
\begin{enumerate}
    \item \label{ass:graphon1} By \Cref{def:rgm} of the graphon , the graphon satisfies $\|W\|_\infty \leq 1$. 
    \item\label{ass:graphon12} \citep[Definition A.10, item 6]{maskey2022generalization}: There exists a constant $\cmin > 0$ such that for every $x \in \chi$, we have $d_W(x) \geq \cmin$.
    \item \label{ass:graphonpair} There exists a constant $d_{\text{cmin}}$ such that for every $x,y\in \mathcal{X}$, we have $\cdW(x,y)\geq d_\text{cmin}$.
    \item \label{asmp:Mtr} $M_\tr = \max(\supp(N^\tr))$ is the largest graph in training, where $N^\tr$ is the distribution of graph sizes in the training data.
    \item\label{ass:graphon4} \citep[Similar to Definition A.10, item 7 adding dependence on $M_\tr$]{maskey2022generalization}: For every $l=1,\ldots,T$, the message function $\Phi^{(l)}$ and update function $\Psi^{(l)}$ are Lipschitz continuous with Lipschitz constants $\LipPhi^{(l)}(M_\tr)$ and $\LipPsi^{(l)}(M_\tr)$ respectively.
\end{enumerate}
\end{assumption}

\section{Large real-world and random graphs have relatively few isomorphic nodes}
\label{appd:large}

In what follows we show that isomorphic nodes are rare both in many real-world networks and SBMs. We start with real-world graphs. 
\citet{macarthur2008symmetry} has computed the fraction of {\em non-isomorphic} nodes (denoted as the {\em network redundancy} $r_\cG$ by \cite{macarthur2008symmetry}) of different types of small (< 23,000 nodes) real-world graphs.
\citet{macarthur2008symmetry} shows that the majority of biological graphs are composed of mostly non-isomorphic nodes. Small technological networks (e.g., road network) tend to have significantly more isomorphic nodes.

In order to see whether these results also hold for larger graphs, we performed a similar experiment on the following datasets.
\begin{itemize}[leftmargin=*]
    \item The ogbl-ppa dataset is an undirected, unweighted graph. Nodes represent proteins from 58 different species, and edges indicate biologically meaningful associations between proteins~\citep{wishart2018drugbank}, e.g., physical interactions, co-expression, homology or genomic neighborhood.
    \item The ogbl-ddi dataset is a homogeneous, unweighted, undirected graph, representing the drug-drug interaction network. Each node represents an FDA-approved or experimental drug~\citep{wishart2018drugbank}. Edges represent interactions between drugs and can be interpreted as a phenomenon where the joint effect of taking the two drugs together is considerably different from the expected effect in which drugs act independently of each other. 
    \item The Slashdot graph contains friend/foe links between the users of Slashdot (where we ignore edge types).
    \item HepPh is a co-authorship network where if an author $i$ co-authored a paper with author $j$, the graph contains a undirected edge from $i$ to $j$. If the paper is co-authored by $k$ authors this generates a completely connected (sub)graph on $k$ nodes.
    \item The Github graph shows GitHub developers (nodes) who have starred at least 10 repositories and edges are mutual follower relationships between them (we make the graph undirected for our analysis).
    \item The Twitch/En graph shows Twitch users (who stream in English) as nodes and links are mutual friendships between them.  
    \item The Epinions graph shows users of the consumer review site Epinions.com as nodes and edges as trust relationships between users.
\end{itemize}   

 \Cref{fig:isonodes} shows the fraction of {\em non-isomorphic} node shown against the size of the graph. This analysis considers a few datasets widely used in the neural network literature to benchmark link prediction methods such as OGB\footnote{\url{https://ogb.stanford.edu/docs/graphprop/}} ppa and ddi, where ppa is the largest dataset we were able to run the {\em nauty}\footnote{https://pallini.di.uniroma1.it/} isomorphism checking algorithm without crashing.
Nauty~\citep{mckay20163} is one of the most efficient graph isomorphism algorithms available, which we use to calculate a lower bound on the size of the automorphism group of our graphs.
Using social networks in the SNAP\footnote{https://snap.stanford.edu/data/index.html} repository we again observe a large fraction of the nodes are non-isomorphic. Visual inspection shows that most isomorphic nodes are low-degree siblings (e.g., the most common are nodes with degree one that have the same parent). 
Note that our results do not contradict~\citet{ball2018symmetric}, which shows that many real-world graphs have isomorphic nodes. Having isomorphic nodes is different than containing a large fraction of isomorphic nodes.

The results in \Cref{fig:isonodes} show that most nodes in real-world graphs tend to be non-isomorphic for reasonably large graphs (in particular ppa). In what follows we show that random graph models generally {\em do not} contain isomorphic nodes with high probability.

\begin{figure}[t!!]
    \centering
    \includegraphics[height=2in,width=3in]{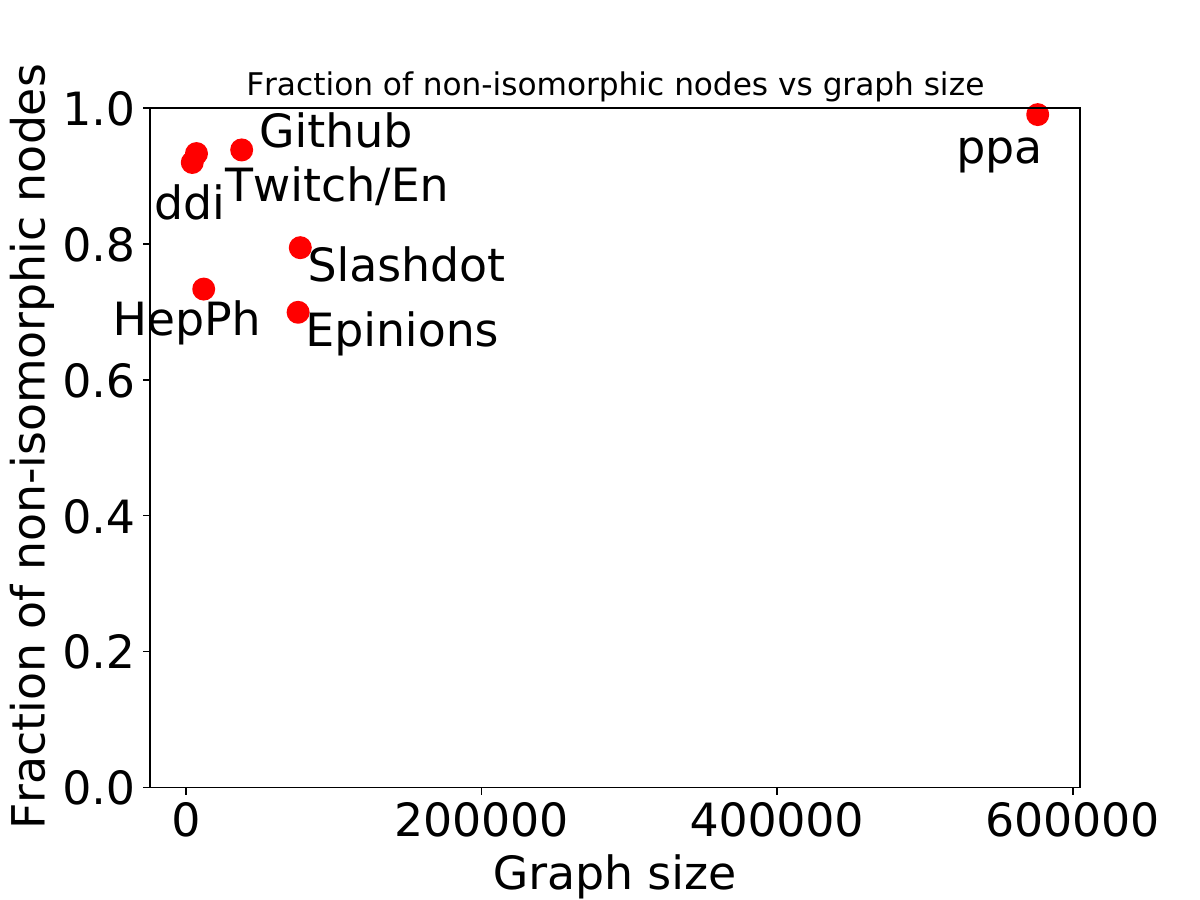}
    \caption{{\bf Fraction of isomorphic nodes in real-world graphs:} The fraction of non-isomorphic nodes (also denoted as the {\em network redundancy} $r_\cG$ by \cite{macarthur2008symmetry}) in real-world graphs tends to be close to 90\%, except in the HepPh collaboration network and Slashdot, which contain many small disconnected components. We assume all graphs are undirected and unattributed for this analysis.}
\vspace{-15pt}
    \label{fig:isonodes}
\end{figure}

{\em Theoretical results on random graphs.} Regarding isomorphic nodes on random graphs, we can prove the following result:
\begin{corollary}\label{cor:ER}
Consider a random graph $G=(V,E)$ with $N$ given nodes so that all possible $2^{N \choose 2}$ graphs should have the same probability to be chosen. 
Then, as $N\to \infty$ all nodes in $G$ are non-isomorphic, regardless whether we take the nodes in $G$ to be attributed or unattributed.
\end{corollary}
\begin{proof}
The proof for unattributed $G$ is a direct consequence of \citet[Theorem~2]{erdos1963asymmetric}.
Adding node attributes cannot make two non-isomorphic nodes be isomorphic, which concludes our proof.
\end{proof}

\citet[Theorem 3]{erdos1963asymmetric} (see also \citet[Theorem 3.1]{kim2002asymmetry}) shows that the statement in \Cref{cor:ER} is also true for $G(N,p)$ graphs with p satisfying $(\ln N)/N \leq p \leq 1-(\ln N)/N$. 
\citet[Theorem 3.1]{kim2002asymmetry} shows a similar result for random $d$-regular graph on $N$ vertices with $3 \leq d \leq n - 4$.
\citet{luczak2019asymmetry} has shown similar results for preferential-attachment graphs, where in each step a new
node with $m \geq 3$ edges is added.
In what follows we show a similar result for SBMs.

\begin{proposition}
\label{prop:noniso}
Consider a random graph $G=(V,E)$ with $N$ given nodes, generated by the SBM in \Cref{def:SBM}, where within-block and inter-block probabilities in $\mS$ lie in the interval $(p,1-p)$, with $(\ln N)/N \leq p \leq 1-(\ln N)/N$. 
Then, as $N\to \infty$ all nodes in $G$ are non-isomorphic, regardless whether we take the nodes in $G$ to be attributed or unattributed.
\end{proposition}
\begin{proof}
Let $\mS$ have $r > 0$ blocks. Consider generating the $G$ first by sampling the within-block edges. Let $G_a$ be the induced subgraph of all nodes that belong to a single block $a \in \{1,\ldots,r\}$.
By the results in \citet[Theorem 3]{erdos1963asymmetric}, as $N\to \infty$, $G_a$ has no isomorphic nodes. 
The above is true for all within-block edges.
Now consider sampling the between-block edges of two $i,j \in V$ nodes in $G$. The event of $i$ and $j$ being isomorphic (if considering just their edges to $G_a$) is the same as they connecting to the same nodes in $G_a$ (since each node on a block is non-isomorphic, if they connect to different nodes they would no longer be isomorphic). The probability of this event is at most $(1 - \epsilon)^{\alpha N}$, for $\epsilon = \min(p,1-p)$, where $\alpha >0$ is the fraction of nodes in $G_a$ (which is not a function of $N$). As $N \to \infty$, by the union bound, the probability that this will happen with any pair of edges is at most ${N \choose 2} (1 - \epsilon)^{\alpha N}$, which goes to zero. 

W.l.o.g.\ now assume $a$ is the block with the least number of nodes (which is also diverging as $N\to \infty$). 
The only alternative for $i$ and $j$ to be isomorphic is to do so by connecting to nodes in distinct blocks. For instance, we could imagine $r$ copies of $G_a$: Even though $i$ and $j$ did not connect to the same nodes in the same graphs, they connected to their isomorphic equivalent nodes in different copies. But since there are only $r$ blocks, and $r$ does not depend on $N$, this event must have probability at most $(1 - \epsilon)^{\alpha N/r}$. As $N \to \infty$, by the union bound, the probability ${N \choose 2} (1 - \epsilon)^{\alpha N/r}$ goes to zero.
Replacing the copies of $G_a$ with the actual sampled blocks only makes this probability smaller, since the subgraphs of the other blocks are larger and may contain different topologies than $G_a$ (making their nodes distinct from the nodes in $G_a$).
Finally, adding node attributes cannot make two non-isomorphic nodes be isomorphic, which concludes our proof.

\end{proof}

\section{Proof results for \Cref{thm:MainInProb-new}}
\label{appd:thm1}

In what follows we provide the elements to prove \Cref{thm:MainInProb-new}.

\CRupdate{First we prove the following lemma of the difference between a graph-graphon (N-normalized) sum aggregation and a sampled-graph (N-normalized) sum aggregation in \Cref{def:graphon-neighbor-avg}. Using the same assumptions as \citet[Lemma B.3]{maskey2022stability}}:

\begin{lemma}
\label{lemma:onA}
Let $(\chi,d, \P) $ be a metric-measure space and
$W$ be a graphon s.t.  Assumptions \ref{ass:graphon}.\ref{ass:graphon1}-\ref{ass:graphon12}. are satisfied.
Let $\Phi: \mathbb{R}^{2F} \to \mathbb{R}^H$ be Lipschitz continuous with Lipschitz constant $\LipPhi(M_\tr)$, with $M_\tr$ as in \Cref{ass:graphon} item 4.
Consider a metric-space signal  $f: \chi \to \mathbb{R}^F$ with $\|f\|_\infty < \infty$.
 Suppose that $X_1, \ldots, X_N$ are drawn i.i.d. from  $\P$ on $\chi$, and let $p \in (0,1/H)$.
Let $x \in \chi $, and define the random variable 
\[
T_x = \frac{1}{N }\sum_{i=1}^N A(x,X_i)\Phi\big(f(x),  f(X_i)\big)
 - \frac{1}{N }\sum_{i=1}^N W(x,X_i)\Phi\big(f(x),  f(X_i)\big)
\]
on the sample space $\mathcal{X}^N\times [0,1]^N$. Then, with probability at least $1-Hp$, we have
\begin{equation}
     \label{eq:HoeffdingsAppOnY_x_new} \|T_x\|_\infty \leq \sqrt{2}\frac{(\LipPhi(M_\tr)\|f\|_\infty +  \|\Phi(0, 0)\|_\infty)\sqrt{\log 2/p} }{\sqrt{N}}.
\end{equation}
\end{lemma}
\CRupdate{
\begin{proof}
The proof of the bound is the same as the proof in \citet[Lemma B.3]{maskey2022stability}, even though $T_x$ is a different quantity than the quantity used in \citet[Lemma B.3]{maskey2022stability}.
\end{proof}
}

\CRupdate{Combining \citet[Lemma B.3]{maskey2022stability} and \Cref{lemma:onA}, we can use the triangle inequality to prove the following lemma about concentration of error between a sampled-graph (N-normalized) sum aggregation in \Cref{def:graphon-neighbor-avg} and the continuous integral aggregation in \Cref{def:graphon-degree-avg}, which is used in \Cref{def:gmpnn,def:cmpnn}. Using the same assumption as \citet[Lemma B.4]{maskey2022stability}}:
\begin{lemma}
\label{lemma:newlemmaA}
Let $(\chi,d, \P) $ be a metric-measure space and
$W$ be a graphon s.t.  Assumptions \ref{ass:graphon}.\ref{ass:graphon1}-\ref{ass:graphon12}. are satisfied.
Let $\Phi: \mathbb{R}^{2F} \to \mathbb{R}^H$ be Lipschitz continuous with Lipschitz constant $\LipPhi(M_\tr)$.
Consider a metric-space signal  $f: \chi \to \mathbb{R}^F$ with $\|f\|_\infty < \infty$.
 Suppose that $X_1, \ldots, X_N$ are drawn i.i.d. from  $\P$ on $\chi$, and let $p \in (0,1/(2H))$.
Let $x \in \chi $, and define the random variable 
\[
R_x = \frac{1}{N }\sum_{i=1}^N A(x,X_i)\Phi\big(f(x),  f(X_i)\big)
 - \int_\chi W(x, y) \Phi\big(f(x),  f(y)\big) d\P(y)
\]
on the sample space $\chi^N\times [0,1]^N$. Then, with probability at least $1-2Hp$, we have
\begin{equation}
     \label{eq:HoeffdingsAppOnY_x_new_tri} \|R_x\|_\infty \leq 2\sqrt{2}\frac{(\LipPhi(M_\tr)\|f\|_\infty +  \|\Phi(0, 0)\|_\infty)\sqrt{\log 2/p} }{\sqrt{N}}.
\end{equation}
\end{lemma}
\begin{proof}
Use the triangle inequality, the results from \citet[Lemma B.3]{maskey2022stability} and \Cref{lemma:onA}. Define $Y_x = \frac{1}{N }\sum_{i=1}^N W(x,X_i)\Phi\big(f(x),  f(X_i)\big)
 - \int_\chi W(x, y) \Phi\big(f(x),  f(y)\big) d\P(y)$.
\[
\|R_x\|_\infty = \|T_x+Y_x\|_\infty\leq \|T_x\|_\infty+\|Y_x\|_\infty.
\]
From \citet[Lemma B.3]{maskey2022stability} and \Cref{lemma:onA}, $\|T_x\|_\infty\leq \sqrt{2}\frac{(\LipPhi(M_\tr)\|f\|_\infty +  \|\Phi(0, 0)\|_\infty)\sqrt{\log 2/p} }{\sqrt{N}}$ w.p. $1-Hp$ and $\|Y_x\|_\infty\leq \sqrt{2}\frac{(\LipPhi(M_\tr)\|f\|_\infty +  \|\Phi(0, 0)\|_\infty)\sqrt{\log 2/p} }{\sqrt{N}}$ w.p. $1-Hp$. We have with probability at least $1-2Hp$ using the union bound of the two events,
\[
\|R_x\|_\infty \leq 2\sqrt{2}\frac{(\LipPhi(M_\tr)\|f\|_\infty +  \|\Phi(0, 0)\|_\infty)\sqrt{\log 2/p} }{\sqrt{N}}.
\]
\end{proof}

Based on \Cref{lemma:newlemmaA}, we can prove the following corollary about the maximum concentration error between sampled-graph (N-normalized) sum aggregation ($\Mone_A$) and continuous integral aggregation ($\Mone_W$) for all the nodes in the sampled graph $G$. \CRupdate{Using the same overall framing as \cite[Lemma B.3]{maskey2022stability}}:
\begin{corollary}
\label{cor:basic}
Consider $(\chi,d, \P) $ a metric-measure space and graphon 
$W$ satisfying \cref{ass:graphon1,ass:graphon12} of \Cref{ass:graphon}.
Let $\Phi: \mathbb{R}^{2F} \to \mathbb{R}^H$ be Lipschitz continuous with Lipschitz constant $\LipPhi(M_\tr)$, and a metric-space signal  $f: \chi \to \mathbb{R}^F$ with $\|f\|_\infty < \infty$.
Define $X_1, \ldots, X_N$ as drawn i.i.d. from  $\P$ on $\chi$, and then edges $A_{i,j}\sim \text{Ber}(W(X_i,X_j))$ i.i.d sampled. Let $p \in (0,1/2H)$, and define the random variable 
\[
R_{X_i} = \frac{1}{N }\sum_{j=1}^N A(X_i,X_j)\Phi\big(f(X_i),  f(X_j)\big)
 - \int_\chi W(X_i, y) \Phi\big(f(X_i),  f(y)\big) d\P(y)
\]
on the sample space $\chi^N\times [0,1]^N$. Then, with probability at least $1-2Hp$, we have %
\begin{equation}
\begin{aligned}
\label{eq:basic-cor} \max_{i=1,...,N}\| (\Mone_A   - \Mone_W) &\big(\Phi (f,f)\big)(X_i) \|_\infty=\max_{i=1,...,N}\|R_{X_i}\|_\infty \\&\leq 2\sqrt{2}\frac{(\LipPhi(M_\tr)\|f\|_\infty +  \|\Phi(0, 0)\|_\infty)\sqrt{\log (2N/p)} }{\sqrt{N}}.
\end{aligned}
\end{equation}
\end{corollary}

\begin{proof}
Using the result from \Cref{lemma:onA} we have with probability $1-\frac{Hp}{N}$,
\[
\begin{aligned}
\|\frac{1}{N }\sum_{i=1}^N &A(x,X_i)\Phi\big(f(x),  f(X_i)\big)
 - \frac{1}{N }\sum_{i=1}^N W(x,X_i)\Phi\big(f(x),  f(X_i)\big)\|_\infty\\&\leq \sqrt{2} \frac{ (\LipPhi(M_\tr)  \|f\|_\infty +  \|\Phi(0, 0)\|_\infty) \sqrt{\log (2N/p) }}{\sqrt{N}}.
\end{aligned}
\]
Using the union bound of the $N$ events that the above equations happens for $x=X_1,...,X_N$, with probability at least $1-Hp$, we have
\[
\begin{aligned}
\max_{i=1,...,N}\|\frac{1}{N }\sum_{j=1}^N &A(X_i,X_j)\Phi\big(f(X_i),  f(X_j)\big)
 - \frac{1}{N }\sum_{j=1}^N W(X_i,X_j)\Phi\big(f(X_i),  f(X_j)\big)\|_\infty\\&\leq \sqrt{2} \frac{ (\LipPhi(M_\tr)  \|f\|_\infty +  \|\Phi(0, 0)\|_\infty) \sqrt{\log (2N/p) }}{\sqrt{N}}.
\end{aligned}
\]
The same logic can be applied to $Y_{X_i}, \forall i\in \{1,...,N\}$.
Thus, using the triangle inequality, and the union bound of the two events, we have with probability at least $1-2Hp$,
$$\max_{i=1,...,N}\|R_{X_i}\|_\infty \leq 2\sqrt{2}\frac{(\LipPhi(M_\tr)\|f\|_\infty +  \|\Phi(0, 0)\|_\infty)\sqrt{\log (2N/p)} }{\sqrt{N}}.$$
\end{proof}

Now the layer-wise error between a \onecMPNN and \onegMPNN can be bounded as follows:

\begin{corollary}
\label{thm:TransfMessages-new-sum}
Consider $(\chi,d, \P) $ a metric-measure space and
graphon $W$ consistent with \cref{ass:graphon1,ass:graphon12} of \Cref{ass:graphon}. Let $\Phi: \mathbb{R}^{2F  } \to \mathbb{R}^H$ and $\Psi: \mathbb{R}^{F + H } \to \mathbb{R}^{F'}$  be Lipschitz continuous with Lipschitz constants $\LipPhi(M_\tr)$ and $\LipPsi(M_\tr)$. Consider a metric-space signal  $f: \chi \to \mathbb{R}^F$ with $\|f\|_\infty < \infty$.  Let  $p \in (0,\frac{1}{2(H+1)})$. Suppose that $X_1, \ldots, X_N$ are drawn i.i.d. from $\P$ in $\chi$, and then edges $A_{i,j}\sim \text{Ber}(W(X_i,X_j))$ i.i.d sampled. Then with probability at least  $1-2Hp$, 
\begin{equation}
\label{eq:lemmab6-1-new-sum}
  \begin{aligned}
& \max_{i=1,...,N}\| \Psi\Big(f(\cdot), \Mone_A\big(\Phi (f,f)\big)(X_i) \Big) - \Psi \Big(f(\cdot) , \Mone_W \big(\Phi (f,f) \big) (X_i) \Big)\|_\infty \\ & \qquad  \leq \LipPsi(M_\tr) \Big(2\sqrt{2}\frac{(\LipPhi(M_\tr)\|f\|_\infty +  \|\Phi(0, 0)\|_\infty)\sqrt{\log (2N/p)} }{\sqrt{N}}\Big).
 \end{aligned} 
\end{equation}
\end{corollary}
\begin{proof}
\CRupdate{The proof is the same as \citet[Lemma B.6]{maskey2022stability}. The different result comes from the different bound between \Cref{cor:basic} and \citet[Lemma B.5]{maskey2022stability}.}
\end{proof}

\subsection{Proof of \Cref{thm:MainInProb-new}}
\label{appx:thm1-proof-sub}
Following \cite[Appendix B.2]{maskey2022stability}, they first bound the layer-wise error as \Cref{thm:TransfMessages-new-sum}, and derive the final bound through a recurrence relation. The only difference is on the layer-wise bound Corollary 6 and \citet[Corollary B.6]{maskey2022stability}. We will omit the middle parts.
Hence, finally, we can prove \Cref{thm:MainInProb-new} by slightly adpating the proof in \citet[Theorem B.14]{maskey2022stability} to our setting.

\thmone*

\begin{proof}

In this case, $\|\Phi^{(l)}(0,0)\|_\infty$, $\|\Psi^{(l)}(0,0)\|_\infty$ can be determined by $(G^\tr,\mF^\tr), N^\tr$ if the MPNN $\Theta$ has been trained on the training graph $(G^\tr,\mF^\tr)$.

Following the procedure of \citet[Appendix B.2]{maskey2022stability} with \Cref{thm:TransfMessages-new-sum}, we can derive similarly, with probability at least $1- \sum_{l=1}^T(2H_l + 1) p$,
\begin{equation}
    \label{eq:main1-2-sum}
    \begin{aligned}
\delta_{\text{A-W}}^{\onenode} \leq &\sum_{l=1}^{T}
\LipPsi^{(l)}(M_\tr)  \Big( 2\sqrt{2}\frac{(\LipPhi(M_\tr)^{(l)}\|f^{\onenode(l)} \|_\infty +  \|\Phi^{(l)}(0, 0)\|_\infty)\sqrt{\log (2N^\te/p)} }{\sqrt{N^\te}}\Big)
\\&\prod_{l' = l+1}^{T}( (\LipPsi^{(l')}(M_\tr))^2 
  + 
   2(\LipPhi^{(l')}(M_\tr))^2 (\LipPsi^{(l')}(M_\tr))^2 ),
\end{aligned}
\end{equation}

\CRupdate{
Using the same proof in \citet[Lemma B.9]{maskey2022stability}, we can derive
\[
||{{\fone}}^{(l)}||_{\infty} \leq {B_1}^{(l)}+{B_2}^{(l)}||f||_{\infty},
\]
where ${B_1}^{(l)}$, ${B_2}^{(l)}$ are independent of $f$, and \begin{equation}
    \label{eq:B'-sum}
    {B_1}^{(l)} = \sum_{k=1}^{l}  \big(
\LipPsi^{(k)}(M_\tr) \|\Phi^{(k)}(0,0)\|_\infty+ \|\Psi^{(k)}(0,0)\|_\infty \big) \prod_{l' = k+1}^{l}  \LipPsi^{(l')}(M_\tr) \big( 1 +   \LipPhi^{(l')}(M_\tr) \big) 
\end{equation}
and
\begin{equation}
    \label{eq:B''-sum}
    {B_2}^{(l)} = \prod_{k = 1}^{l} \LipPsi^{(k)}(M_\tr) \left(1  +  \LipPhi^{(k)}(M_\tr) \right).
\end{equation}
}

Now we can decompose the summation in \Cref{eq:main1-2-sum}. First, we defince $C_1$ as
\begin{equation}
\label{eq:c1-sum}
\begin{aligned}
C_1 &= \sum_{l=1}^{T}
\LipPsi^{(l)}(M_\tr) \Big(2\sqrt{2}{(\LipPhi^{(l)}}(M_\tr) \; {B_1}^{(l)}+  \|\Phi^{(l+1)}(0, 0)\|_\infty)\Big) \\
&\quad \times \prod_{l' = l+1}^{T} ((\LipPsi^{(l')}(M_\tr))^2 
  + 
   2(\LipPhi^{(l')}(M_\tr))^2 (\LipPsi^{(l')}(M_\tr))^2) ,
 \end{aligned}
\end{equation}
Then we can define $C_2$ as
\begin{equation}
\label{eq:c2-sum}
\begin{aligned}
C_2 &= \sum_{l=1}^{T}
\LipPsi^{(l)}(M_\tr) \Big( 2\sqrt{2}\LipPhi^{(l)}(M_\tr) {B_2}^{(l)}\Big)
\prod_{l' = l+1}^{T} ((\LipPsi^{(l')}(M_\tr))^2 
  + 
   2(\LipPhi^{(l')}(M_\tr))^2 (\LipPsi^{(l')}(M_\tr))^2 ),
\end{aligned}
\end{equation}
It is clear to see we can rewrite \Cref{eq:main1-2-sum} as
\begin{equation}
    \delta_{\text{A-W}}^{\onenode}
\leq (C_1+C_2\|f\|_\infty)\frac{\sqrt{\log (2N^\te/p)}}{\sqrt{N^\te}}.
\end{equation}
Thus $C_1$ and $C_2$ depends on $\{\LipPhi^{(l)}(M_\tr)\}_{l=1}^T$ and $\{\LipPsi^{(l)}(M_\tr)\}_{l=1}^T$.
\end{proof}  %
\section{Proof of theoretical results for hardness of link prediction}
\label{appd:link-proof}
In this section, we prove the results for $\ThetaoneTcap_A$ and $\ThetaoneTcap_W$ based on \Cref{thm:MainInProb-new}. 
Now we can prove \Cref{cor:stab-node}.

\corone*

\begin{proof}
The proof follows \Cref{thm:MainInProb-new} by using the triangle inequality.

From \Cref{thm:MainInProb-new}, we know with probability at least $1- 2\sum_{l=1}^T (H_l + 1) p$, $\|\ThetaoneTcap_{A^\te}(\mF^\te)_i-\ThetaoneTcap_W(f)(X^\te_i)\|_\infty\leq (C_1+C_2\|f\|_\infty)\frac{\sqrt{\log 2N^\te/p}}{\sqrt{N^\te}}$ and $\|\ThetaoneTcap_{A^\te}(\mF^\te)_j-\ThetaoneTcap_W(f)(X^\te_j)\|_\infty\leq (C_1+C_2\|f\|_\infty)\frac{\sqrt{\log 2N^\te/p}}{\sqrt{N^\te}}$.

Then
\[
\begin{aligned}
&\|\ThetaoneTcap_{A^\te}(\mF^\te)_i-\ThetaoneTcap_{A^\te}(\mF^\te)_j\|_\infty\\&\leq \|\ThetaoneTcap_{A^\te}(\mF^\te)_i-\ThetaoneTcap_W(f)(X^\te_i)\|_\infty+\|\ThetaoneTcap_W(f)(X^\te_i)-\ThetaoneTcap_{A^\te}(\mF^\te)_j\|_\infty\\&= \|\ThetaoneTcap_{A^\te}(\mF^\te)_i-\ThetaoneTcap_W(f)(X^\te_i)\|_\infty+\|\ThetaoneTcap_W(f)(X^\te_j)-\ThetaoneTcap_{A^\te}(\mF^\te)_j\|_\infty\\&\leq (C_1+C_2\|f\|_\infty)\frac{2\sqrt{\log (2N^\te/p)}}{\sqrt{N^\te}}.
\end{aligned}
\]
The first inequality holds by traingle inequality, and the second equation holds since $\ThetaoneTcap_W(f)(X^\te_i)=\ThetaoneTcap_W(f)(X^\te_j)$.
\end{proof}

Then we are able to prove \Cref{lem:cmpnn-sbm} by induction. By our \Cref{def:iso-sbm}, we can also claim $t_k-t_{k-1}=t_{\pi(k)}-t_{{\pi(k)}-1}, \forall k\in \{1,...,r\}$.

\lemone*

\begin{proof}
We prove the lemma by induction. 

We assume in layer $l$,  ${\fone}^{(l)}(X^\te_i)={\fone}^{(l)}(X^\te_j), 1\leq l\leq T-1$, ${\fone}^{(l)}$ outputs the same value within each block $\mB^{(l)}$, and $\mB^{(l)}=\pi\circ \mB^{(l)}$. By \Cref{def:SBM,def:iso-sbm}, we know the assumption holds for $l=1$. First,
\[
\begin{aligned}
{\fone}^{(l+1)}(X^\te_i)&=\Psi^{(l+1)} \Big({\fone}^{(l)} (X^\te_i) , \Mone_W\big(\Phi^{(l+1)} ({\fone}^{(l)},{\fone}^{(l)}) \big) (X^\te_i) \Big).
\end{aligned}
\]
Since ${\fone}^{(l)}(X^\te_i)={\fone}^{(l)}(X^\te_j)$, we only need to show $\Mone_W\big(\Phi^{(l+1)} ({\fone}^{(l)},{\fone}^{(l)}) \big) (X^\te_i)=\Mone_W\big(\Phi^{(l+1)} ({\fone}^{(l)},{\fone}^{(l)}) \big) (X^\te_j)$.

\begin{equation}
\begin{aligned}
&\Mone_W\big(\Phi^{(l+1)} ({\fone}^{(l)},{\fone}^{(l)}) \big) (X^\te_i)\\&=\int_{[0,1]} W(X^\te_i,y)\Phi^{(l+1)}({\fone}^{(l)}(X^\te_i),{\fone}^{(l)}(y))dy\\&= \sum_{k=1}^r \int_{[t_{k-1},t_{k})}W(X^\te_i,y)\Phi^{(l+1)}({\fone}^{(l)}(X^\te_i),{\fone}^{(l)}(y))dy\\&=\sum_{k=1}^r \Phi^{(l+1)}(\mB^{(l)}_a,\mB^{(l)}_k) \int_{[t_{k-1},t_{k})}W(X^\te_i,y)dy\\&=\sum_{k=1}^r \Phi^{(l+1)}(\mB^{(l)}_a,\mB^{(l)}_k) (t_{k}-t_{k-1})\mS_{i,k}\\&=\sum_{k=1}^r \Phi^{(l+1)}(\mB^{(l)}_a,\mB^{(l)}_k) (t_{k}-t_{k-1})\mS_{\pi(i),\pi(k)}
\end{aligned}
\end{equation}
\begin{equation}
\begin{aligned}
&=\sum_{k=1}^r \Phi^{(l+1)}(\mB^{(l)}_{\pi(a)},\mB^{(l)}_{\pi(k)}) (t_{\pi(k)}-t_{\pi(k)-1})\mS_{\pi(i),\pi(k)}\\&=\sum_{k=1}^r \Phi^{(l+1)}(\mB^{(l)}_b,\mB^{(l)}_{\pi(k)}) (t_{\pi(k)}-t_{\pi(k)-1})\mS_{j,\pi(k)}\\&=\sum_{k=1}^r \Phi^{(l+1)}(\mB^{(l)}_b,\mB^{(l)}_{k}) (t_{k}-t_{k-1})\mS_{j,k}\\&=\sum_{k=1}^r \Phi^{(l+1)}(\mB^{(l)}_b,\mB^{(l)}_k) \int_{[t_{k-1},t_{k})}W(X^\te_j,y)dy\\&=\int_{[0,1]} W(X^\te_j,y)\Phi^{(l+1)}({\fone}^{(l)}(X^\te_j),{\fone}^{(l)}(y))dy\\&=\Mone_W\big(\Phi^{(l+1)} ({\fone}^{(l)},{\fone}^{(l)}) \big) (X^\te_j)
\end{aligned}
\end{equation}

Here we use the fact that ${\fone}^{(l)}$ ${\fone}^{(l)}$ outputs the same value within each block, and $\mB^{(l)}_k=\mB^{(l)}_{\pi(k)}, \forall k\in\{1,...,r\}$.

We have shown ${\fone}^{(l+1)}(X^\te_i)={\fone}^{(l+1)}(X^\te_j)$. And this proof applies for all $ X^\te_i\in [t_{a-1},t_a)$ (in block $a$), and the same conclusion holds. So ${\fone}^{(l+1)}$ still outputs the same value within each block. Furthermore, $\mB^{(l+1)}_a = \mB^{(l+1)}_{\pi(a)}$ using the same proof technique. And this implies $\pi\circ \mB^{(l+1)}=\mB^{(l+1)}$.

Thus, $\ThetaoneTcap_W(X^\te_i)={\fone}^{(T)}(X^\te_i) = {\fone}^{(T)}(X^\te_j) = \ThetaoneTcap_W(X^\te_j)$.
\end{proof}

Then we are ready to prove \Cref{cor:perf} by applying \Cref{cor:stab-node}. %

\cortwo*

\begin{proof}
To prove the corollary, we assume we have two nodes $j$ and $j'$, such that $i$ and $j$ are in the same block while $i$ and $j'$ are in distinct isomorphic blocks. In the proof, we will show that the link prediction between $i$ and $j$ and the prediction between $i$ and $j'$ will be the same.

First, from \Cref{cor:stab-node}, since nodes $j$ and $j'$ are in distinct isomorphic SBM blocks, when \Cref{eq:CondOnN-new} is satisfied, we have with probability at least $1- 2\sum_{l=1}^T (H_l + 1) p$
\[
\|\ThetaoneTcap_{A^\te}(\mF^\te)_j-\ThetaoneTcap_{A^\te}(\mF^\te)_{j'}\|_\infty \leq (C_1+C_2\|f\|_\infty)\frac{2\sqrt{\log 2N^\te/p}}{\sqrt{N^\te}}.
\]
Then when the requirement for $N^\te$ is satisfied,
\[
\begin{aligned}
 &\|\etaone(\ThetaoneTcap_{A^\te}(\mF^\te)_i,\ThetaoneTcap_{A^\te}(\mF^\te)_j)- \etaone(\ThetaoneTcap_{A^\te}(\mF^\te)_i,\ThetaoneTcap_{A^\te}(\mF^\te)_{j'})\|_\infty\\&\leq L_\etaone(M_\tr) \|\ThetaoneTcap_{A^\te}(\mF^\te)_j-\ThetaoneTcap_{A^\te}(\mF^\te)_{j'}\|_\infty\leq L_\etaone(M_\tr) (C_1+C_2\|f\|_\infty)\frac{2\sqrt{\log 2N^\te/p}}{\sqrt{N^\te}}\\&< \|\etaone(\ThetaoneTcap_{A^\te}(\mF^\te)_i,\ThetaoneTcap_{A^\te}(\mF^\te)_j)-\tau\|_\infty
\end{aligned}
\]

If $\etaone(\ThetaoneTcap_{A^\te}(\mF^\te)_i,\ThetaoneTcap_{A^\te}(\mF^\te)_j)>\tau$, then
\[
\begin{aligned}
 &\etaone(\ThetaoneTcap_{A^\te}(\mF^\te)_i,\ThetaoneTcap_{A^\te}(\mF^\te)_{j'})\\&\geq \etaone(\ThetaoneTcap_{A^\te}(\mF^\te)_i,\ThetaoneTcap_{A^\te}(\mF^\te)_j) - |\etaone(\ThetaoneTcap_{A^\te}(\mF^\te)_i,\ThetaoneTcap_{A^\te}(\mF^\te)_j) -\etaone(\ThetaoneTcap_{A^\te}(\mF^\te)_i,\ThetaoneTcap_{A^\te}(\mF^\te)_{j'})|\\&> \etaone(\ThetaoneTcap_{A^\te}(\mF^\te)_i,\ThetaoneTcap_{A^\te}(\mF^\te)_j) - |\etaone(\ThetaoneTcap_{A^\te}(\mF^\te)_i,\ThetaoneTcap_{A^\te}(\mF^\te)_{j}) -\tau| \\&= \etaone(\ThetaoneTcap_{A^\te}(\mF^\te)_i,\ThetaoneTcap_{A^\te}(\mF^\te)_j) - (\etaone(\ThetaoneTcap_{A^\te}(\mF^\te)_i,\ThetaoneTcap_{A^\te}(\mF^\te)_j) -\tau)=\tau
\end{aligned}
\]

If $\etaone(\ThetaoneTcap_{A^\te}(\mF^\te)_i,\ThetaoneTcap_{A^\te}(\mF^\te)_j)<\tau$, then
\[
\begin{aligned}
 &\etaone(\ThetaoneTcap_{A^\te}(\mF^\te)_i,\ThetaoneTcap_{A^\te}(\mF^\te)_{j'})\\&\leq \etaone(\ThetaoneTcap_{A^\te}(\mF^\te)_i,\ThetaoneTcap_{A^\te}(\mF^\te)_j) + |\etaone(\ThetaoneTcap_{A^\te}(\mF^\te)_i,\ThetaoneTcap_{A^\te}(\mF^\te)_j) -\etaone(\ThetaoneTcap_{A^\te}(\mF^\te)_i,\ThetaoneTcap_{A^\te}(\mF^\te)_{j'})|\\&< \etaone(\ThetaoneTcap_{A^\te}(\mF^\te)_i,\ThetaoneTcap_{A^\te}(\mF^\te)_j) + |\etaone(\ThetaoneTcap_{A^\te}(\mF^\te)_i,\ThetaoneTcap_{A^\te}(\mF^\te)_j) -\tau| \\&= \etaone(\ThetaoneTcap_{A^\te}(\mF^\te)_i,\ThetaoneTcap_{A^\te}(\mF^\te)_j) - (\etaone(\ThetaoneTcap_{A^\te}(\mF^\te)_i,\ThetaoneTcap_{A^\te}(\mF^\te)_j) -\tau)=\tau
\end{aligned}
\]
So whether $i$ and $j$ are in the same block, or in distinct isomorphic SBM blocks, their prediction will be the same (both links have predictions larger than $\tau$ or less).
\end{proof}

\section{Proof for pairwise \twogMPNN and \twocMPNN}
\label{appd:pair}
First we prove \Cref{lem:stationary} showing $W(x,y)$ is a stationary point in \twocMPNN.

\lemtwo*
\begin{proof}
If ${f^\pairwisenodes}^{(t-1)}(x,y)=W(x,y)$, then
\[
\begin{aligned}
 \Mtwo_W(\Phi^{(t)}(f^{(t-1)}))(x,y)&=\frac{1}{2} \int_\mathcal{X} (\frac{W(y,z)}{c_W(x,y)}\Phi^{(t)}({f^\pairwisenodes}^{(t-1)}(x,y),{f^\pairwisenodes}^{(t-1)}(x,z)) \\&+ \frac{W(x,z)}{c_W(x,y)}\Phi^{(t)}({f^\pairwisenodes}^{(t-1)}(x,y),{f^\pairwisenodes}^{(t-1)}(y,z))) d\mu(z)\\&=\frac{1}{2} \int_\mathcal{X} (\frac{W(y,z)}{c_W(x,y)}W(x,z) + \frac{W(x,z)}{c_W(x,y)}W(y,z)) d\mu(z)\\&=\frac{1}{c_W(x,y)}\int_\mathcal{X} W(x,z)W(y,z)d\mu(z) \\&= \frac{c_W(x,y)}{c_W(x,y)}=1
\end{aligned}
\]

Thus ${f^\pairwisenodes}^{(t)}(x,y)=\Psi^{(t)}({f^\pairwisenodes}^{(t-1)}(x,y),\Mtwo_W(\Phi^{(t)}({f^\pairwisenodes}^{(t-1)},{f^\pairwisenodes}^{(t-1)} ))(x,y))=\frac{W(x,y)}{1}=W(x,y)$. %

We finish proving $W(x,y)$ is a stanionary point in \twocMPNN. There are infinity choices of $\Phi$ and $\Psi$ such that $W(x,y)$ is a stanionary point.
\end{proof}

Then we aim to prove \Cref{thm:MainInProb-new2}, and the prove procedure should be very similar as \Cref{thm:MainInProb-new}.

\subsection{Preparation}

\CRupdate{Following \citet[Lemma B.3]{maskey2022stability}, we propose the following lemma for \twocMPNN. Using the same overall framing as \citet[Lemma B.3]{maskey2022stability}},

\begin{lemma}
\label{lemma:3-new}
Let $(\chi,d, \P) $ be a metric-measure space and
$W$ be a graphon s.t.  Assumptions \ref{ass:graphon}.\ref{ass:graphon1}-\ref{ass:graphonpair}. are satisfied.
Let $\Phi: \mathbb{R}^{2F} \to \mathbb{R}^H$ be Lipschitz continuous with Lipschitz constant $\LipPhi(M_\tr)$.
Consider a metric-space signal  $f^\pairwisenodes: \chi\times \chi \to \mathbb{R}^F$ with $\|f^\pairwisenodes\|_\infty < \infty$.
 Suppose that $X_1, \ldots, X_N$ are drawn i.i.d. from  $\P$ on $\chi$, and let $p \in (0,1/H)$.
Let $x,y \in \chi $, and define the random variable 
\[
\begin{aligned}
 Y_{x,y}^\pairwisenodes &= \frac{1}{2N}\sum_{i=1}^N \Big(W(y,X_i)\Phi(f^\pairwisenodes(x,y), f^\pairwisenodes(x,X_i))+W(x,X_i)\Phi(f^\pairwisenodes(x,y), f^\pairwisenodes(y,X_i))\Big)\\&-\frac{1}{2}\int_{\mathcal{X}} \Big(W(y,z)\Phi(f^\pairwisenodes(x,y), f^\pairwisenodes(x,z)))+W(x,z)\Phi(f^\pairwisenodes(x,y), f^\pairwisenodes(y,z))\Big)d\mu(z)
\end{aligned}
\]
on the sample space $\chi^N$. Then, with probability at least $1-Hp$, we have
\begin{equation}
     \label{eq:HoeffdingsAppOnY_x-new} \|Y_{x,y}^\pairwisenodes\|_\infty \leq \sqrt{2}\frac{(\LipPhi(M_\tr)\|f^\pairwisenodes\|_\infty +  \|\Phi(0, 0)\|_\infty)\sqrt{\log 2/p} }{\sqrt{N}}.
\end{equation}
\end{lemma}

\CRupdate{
\begin{proof}
The proof of the bound is the same as the proof in \citet[Lemma B.3]{maskey2022stability}, even though $Y_{x,y}$ is a different quantity than the quantity used in \citet[Lemma B.3]{maskey2022stability}.
\end{proof}
}

\begin{lemma}
\label{lemma:onA-new}
Let $(\chi,d, \P) $ be a metric-measure space and
$W$ be a graphon s.t.  Assumptions \ref{ass:graphon}.\ref{ass:graphon1}-\ref{ass:graphonpair}. are satisfied.
Let $\Phi: \mathbb{R}^{2F} \to \mathbb{R}^H$ be Lipschitz continuous with Lipschitz constant $\LipPhi(M_\tr)$.
Consider a metric-space signal  $f^\pairwisenodes: \chi\times \chi \to \mathbb{R}^F$ with $\|f^\pairwisenodes\|_\infty  < \infty$.
 Suppose that $X_1, \ldots, X_N$ are drawn i.i.d. from  $\P$ on $\chi$, and let $p \in (0,1/H)$.
Let $x,y \in \chi $, and define the random variable 
\[
\begin{aligned}
T^\pairwisenodes_{x,y}& = \frac{1}{2N}\sum_{i=1}^N \Big(A(y,X_i)\Phi(f^\pairwisenodes(x,y), f^\pairwisenodes(x,X_i))+A(x,X_i)\Phi(f^\pairwisenodes(x,y), f^\pairwisenodes(y,X_i))\Big)\\&
 - \frac{1}{2N}\sum_{i=1}^N \Big(W(y,X_i)\Phi(f^\pairwisenodes(x,y), f^\pairwisenodes(x,X_i))+W(x,X_i)\Phi(f^\pairwisenodes(x,y), f^\pairwisenodes(y,X_i))\Big)
\end{aligned}
\]
on the sample space $\mathcal{X}^N\times [0,1]^{2N}$. Then, with probability at least $1-Hp$, we have
\begin{equation}
     \label{eq:HoeffdingsAppOnY_x_new-new} \|T^\pairwisenodes_{x,y}\|_\infty \leq \sqrt{2}\frac{(\LipPhi(M_\tr)\|f^\pairwisenodes\|_\infty  +  \|\Phi(0, 0)\|_\infty)\sqrt{\log 2/p} }{\sqrt{N}}.
\end{equation}
\end{lemma}
\CRupdate{
\begin{proof}
The proof procedure is the same as \citet[Lemma B.3]{maskey2022stability} where we use $\mathbb{E}[A(y,X_i)]=W(y,X_i)$ and $\mathbb{E}[A(x,X_i)]=W(x,X_i)$.
\end{proof}
}

\begin{lemma}
\label{lemma:newlemmaA-new}
Let $(\chi,d, \P) $ be a metric-measure space and
$W$ be a graphon s.t.  Assumptions \ref{ass:graphon}.\ref{ass:graphon1}-\ref{ass:graphonpair}. are satisfied.
Let $\Phi: \mathbb{R}^{2F} \to \mathbb{R}^H$ be Lipschitz continuous with Lipschitz constant $\LipPhi(M_\tr)$.
Consider a metric-space signal  $f^\pairwisenodes: \chi\times \chi \to \mathbb{R}^F$ with $\|f^\pairwisenodes\|_\infty < \infty$.
 Suppose that $X_1, \ldots, X_N$ are drawn i.i.d. from  $\P$ on $\chi$, and let $p \in (0,1/(2H))$.
Let $x,y \in \chi $, and define the random variable 
\[
\begin{aligned}
R^\pairwisenodes_{x,y}& = \frac{1}{2N}\sum_{i=1}^N \Big(A(y,X_i)\Phi(f^\pairwisenodes(x,y), f^\pairwisenodes(x,X_i))+A(x,X_i)\Phi(f^\pairwisenodes(x,y), f^\pairwisenodes(y,X_i))\Big)\\&-\frac{1}{2}\int_{\mathcal{X}} \Big(W(y,z)\Phi(f^\pairwisenodes(x,y), f^\pairwisenodes(x,z)))+W(x,z)\Phi(f^\pairwisenodes(x,y), f^\pairwisenodes(y,z))\Big)d\mu(z)
\end{aligned}
\]
on the sample space $\chi^N\times [0,1]^{2N}$. Then, with probability at least $1-2Hp$, we have
\begin{equation}
     \label{eq:HoeffdingsAppOnY_x_new_tri-new} \|R^\pairwisenodes_{x,y}\|_\infty \leq 2\sqrt{2}\frac{(\LipPhi(M_\tr)\|f^\pairwisenodes\|_\infty +  \|\Phi(0, 0)\|_\infty)\sqrt{\log 2/p} }{\sqrt{N}}.
\end{equation}
\end{lemma}
\begin{proof}
Use the triangle inequality and the results from \Cref{lemma:3-new,lemma:onA-new}.
\[
\|R^\pairwisenodes_{x,y}\|_\infty = \|T^\pairwisenodes_{x,y}+Y^\pairwisenodes_{x,y}\|_\infty\leq \|T^\pairwisenodes_{x,y}\|_\infty+\|Y^\pairwisenodes_{x,y}\|_\infty.
\]
From \Cref{lemma:3-new,lemma:onA-new}, $\|T^\pairwisenodes_{x,y}\|_\infty\leq \sqrt{2}\frac{(\LipPhi(M_\tr)\|f^\pairwisenodes\|_\infty +  \|\Phi(0, 0)\|_\infty)\sqrt{\log 2/p} }{\sqrt{N}}$ w.p. $1-Hp$ and $\|Y^\pairwisenodes_{x,y}\|_\infty\leq \sqrt{2}\frac{(\LipPhi(M_\tr)\|f^\pairwisenodes\|_\infty +  \|\Phi(0, 0)\|_\infty)\sqrt{\log 2/p} }{\sqrt{N}}$ w.p. $1-Hp$. With probability at least $1-2Hp$, by intersecting the two events, we have
\[
\|R^\pairwisenodes_{x,y}\|_\infty \leq 2\sqrt{2}\frac{(\LipPhi(M_\tr)\|f^\pairwisenodes\|_\infty +  \|\Phi(0, 0)\|_\infty)\sqrt{\log 2/p} }{\sqrt{N}}.
\]
\end{proof}

Based on \Cref{lemma:newlemmaA-new}, we can prove the following corollary about the maximum concentration error for all pairs of nodes in the sampled graph $G$. \CRupdate{Using the same overall framing as \citet[Lemma B.3]{maskey2022stability},}
\begin{corollary}
\label{cor:basic-new}
Let $(\chi,d, \P) $ be a metric-measure space and
$W$ be a graphon s.t.  Assumptions \ref{ass:graphon}.\ref{ass:graphon1}-\ref{ass:graphonpair}. are satisfied.
Let $\Phi: \mathbb{R}^{2F} \to \mathbb{R}^H$ be Lipschitz continuous with Lipschitz constant $\LipPhi(M_\tr)$.
Consider a metric-space signal  $f^\pairwisenodes: \chi\times \chi \to \mathbb{R}^F$ with $\|f^\pairwisenodes\|_\infty  < \infty$.
 Suppose that $X_1, \ldots, X_N$ are drawn i.i.d. from  $\P$ on $\chi$, and then edges $A_{i,j}\sim \text{Ber}(W(X_i,X_j))$ i.i.d sampled. Let $p \in (0,1/2H)$, and define the random variable 
\[
\begin{aligned}
R^\pairwisenodes_{X_i,X_j}& = \frac{1}{2N}\sum_{z=1}^N \Big(A(X_j,X_z)\Phi(f^\pairwisenodes(X_i,X_j), f^\pairwisenodes(X_i,X_z))\\&+A(X_i,X_z)\Phi(f^\pairwisenodes(X_i,X_j), f^\pairwisenodes(X_j,X_z))\Big)\\&-\frac{1}{2}\int_{\mathcal{X}} \Big(W(X_j,z)\Phi(f^\pairwisenodes(X_i,X_j), f^\pairwisenodes(X_i,z)))\\&+W(X_i,z)\Phi(f^\pairwisenodes(X_i,X_j), f^\pairwisenodes(X_j,z))\Big)d\mu(z)
\end{aligned}
\]
on the sample space $\chi^N\times [0,1]^{2N}$. Then, with probability at least $1-2Hp$, we have %
\begin{equation}
\begin{aligned}
\label{eq:basic-cor-new} \max_{i,j=1,...,N}\|R^\pairwisenodes_{X_i,X_j}\|_\infty \leq 2\sqrt{2}\frac{(\LipPhi(M_\tr)\|f^\pairwisenodes\|_\infty  +  \|\Phi(0, 0)\|_\infty)\sqrt{\log (2N^2/p)} }{\sqrt{N}}.
\end{aligned}
\end{equation}
\end{corollary}

\begin{proof}
Using the result from \Cref{lemma:onA-new}, we have with probability $1-\frac{Hp}{N^2}$,
\[
\begin{aligned}
\|T^\pairwisenodes_{X_i,X_j}\|_\infty\leq \sqrt{2} \frac{ (\LipPhi(M_\tr)  \|f^\pairwisenodes\|_\infty  +  \|\Phi(0, 0)\|_\infty) \sqrt{\log (2N^2/p) }}{\sqrt{N}}.
\end{aligned}
\]
Using the union bound of the $N^2$ events that the above equations happens for $x=X_1,...,X_N$ and $y=X_1,...,X_N$, with probability at least $1-Hp$,  we have %
\[
\begin{aligned}
\max_{i,j=1,...,N}\|T^\pairwisenodes_{X_i,X_j}\|_\infty\leq \sqrt{2} \frac{ (\LipPhi(M_\tr)  \|f^\pairwisenodes\|_\infty  +  \|\Phi(0, 0)\|_\infty) \sqrt{\log (2N^2/p) }}{\sqrt{N}}.
\end{aligned}
\]
The same logic can be applied to $Y^\pairwisenodes_{X_i,X_j}, \forall i\in \{1,...,N\}$.
Thus, using the triangle inequality, and the union bound of the two events, we have with probability at least $1-2Hp$,
$$\max_{i,j=1,...,N}\|R^\pairwisenodes_{X_i,X_j}\|_\infty \leq 2\sqrt{2}\frac{(\LipPhi(M_\tr)\|f^\pairwisenodes\|_\infty  +  \|\Phi(0, 0)\|_\infty)\sqrt{\log (2N^2/p)} }{\sqrt{N}}.$$
\end{proof}

\CRupdate{Following \citet[Lemma B.2]{maskey2022stability}, we also bound the maximum sampled-graph fraction of common neighbors $c_A(\cdot,\cdot)$ under a
condition of the sample size $N$. Using a same assumption as \citet[Lemma B.2]{maskey2022stability},}
\begin{lemma}
\label{cor:newdegreebound-new}
Let $(\chi,d, \P) $ be a metric-measure space and
$W$ be a graphon s.t.  Assumptions \ref{ass:graphon}.\ref{ass:graphon1}-\ref{ass:graphonpair}. are satisfied.
 Suppose that $X_1, \ldots, X_N$ are drawn i.i.d. from  $\P$ on $\chi$, and then edges $A_{i,j}\sim \text{Ber}(W(X_i,X_j))$ i.i.d sampled. And let $p \in (0,1)$.
If $N \in \mathbb{N}$ satisfy
\begin{equation}
\label{eq:largeN-new-new}
    \sqrt{N} \geq  4\sqrt{2}\frac{\sqrt{\log{(2N^2/p)}}}{d_\text{cmin}}.
\end{equation}
Then, with probability at least $1-2p$, we have
\[
\max_{i,j=1,...,N}\|c_A(X_i,X_j)-c_W(X_i,X_j)\|_\infty \leq 2\sqrt{2}\frac{\sqrt{\log{(2N^2/p)}}}{\sqrt{N}} ,
\]
and
\begin{equation}
     \label{eq:HoeffdingsAppOnY_x_new_degree-new} \min_{i,j=1,...,N} c_A(X_i,X_j) \geq \frac{d_\text{cmin}}{2}.
\end{equation}

\end{lemma}

\begin{proof}
For given $x,y\in\mathcal{X}$, define the random variable
\[
c_A(x,y)-c_X(x,y) = \frac{1}{N }\sum_{i=1}^N A(x,X_i)A(y,X_i)
 - \frac{1}{N }\sum_{i=1}^N W(x,X_i)W(y,X_i)
\]
on the sample space $\chi^N\times[0,1]^{2N}$. Using the same proof technique in \Cref{lemma:3-new,lemma:onA-new,lemma:newlemmaA-new}, we can prove with probability at least $1-2p$, we have %
\[
\max_{i,j=1,...,N}\|c_A(X_i,X_j)-c_W(X_i,X_j)\|_\infty \leq 2\sqrt{2}\frac{\sqrt{\log{(2N^2/p)}}}{\sqrt{N}} ,
\]
Since $c_W(X_i,X_j)\geq d_\text{cmin}$, then with probability at least $1-2p$, when \Cref{eq:largeN-new-new} is satisfied, we have %
\[
\min_{i,j=1,...,N}c_A(X_i,X_j) \geq \frac{d_\text{cmin}}{2}
\]
\end{proof}
\CRupdate{Based on \Cref{cor:newdegreebound-new}, we can prove a modified version of \citet[Lemma B.5]{maskey2022stability}. Using a same overall framing as \citet[Lemma B.5]{maskey2022stability},}

\begin{lemma}
\label{lemma:TransfMessage-new}
Let $(\chi,d, \P) $ be a metric-measure space and
$W$ be a graphon s.t.  Assumptions \ref{ass:graphon}.\ref{ass:graphon1}-\ref{ass:graphonpair}. are satisfied. Let $\Phi: \mathbb{R}^{2F  } \to \mathbb{R}^H$  be Lipschitz continuous with Lipschitz constant $\LipPhi(M_\tr)$. Consider a metric-space signal  $f^\pairwisenodes: \chi\times \chi \to \mathbb{R}^F$ with $\|f^\pairwisenodes\|_\infty < \infty$. Let $p \in (0,\frac{1}{2(H+1)})$, 
and let $N \in \mathbb{N}$ satisfy (\ref{eq:largeN-new-new}). Suppose that $X_1, \ldots, X_N$ are drawn i.i.d. from  $\P$ in $\chi$, and then edges $A_{i,j}\sim \text{Ber}(W(X_i,X_j))$ i.i.d sampled.
Then, condition (\ref{eq:HoeffdingsAppOnY_x_new_degree-new}) together with (\ref{eq:lemmab5-1-new}) below are satisfied in probability at least $1-2(H + 1)p$: 
\begin{equation}
\begin{aligned}
\label{eq:lemmab5-1-new}
  &\max_{i,j=1,...,N}\| (\Mtwo_A   - \Mtwo_W) \big(\Phi (f^\pairwisenodes,f^\pairwisenodes)\big)(X_i,X_j) \|_\infty \\  &\leq 4 \frac{\sqrt{2}\sqrt{\log (2N^2/P)}}{\sqrt{N} d_\text{cmin}^2}(\LipPhi(M_\tr)\|f^\pairwisenodes\|_\infty+ \|\Phi(0, 0)\|_\infty )\\ &+ \frac{2\sqrt{2}(\LipPhi(M_\tr)\|f^\pairwisenodes\|_\infty+  \|\Phi(0, 0)\|_\infty) \sqrt{\log (2N^2/P)}}{d_\text{cmin}\sqrt{N}},
\end{aligned}
\end{equation}
\end{lemma}

\CRupdate{
\begin{proof}
The proof is the same as \citet[Lemma B.5]{maskey2022stability}. The only difference is on the difference between \Cref{cor:newdegreebound-new} and \citet[Lemma B.2]{maskey2022stability}, and the difference between \Cref{cor:basic-new} and \citet[Lemma B.4]{maskey2022stability}.
\end{proof}

Same as \citet[Lemma B.6]{maskey2022stability}, the layer-wise error for \twocMPNN and a \twogMPNN can be bounded. Using the same overall framing as \citet[Lemma B.6]{maskey2022stability},}

\begin{corollary}
\label{thm:TransfMessages-new-new}
Let $(\chi,d, \P) $ be a metric-measure space and
$W$ be a graphon s.t.  Assumptions \ref{ass:graphon}.\ref{ass:graphon1}-\ref{ass:graphonpair}. are satisfied. Let $\Phi: \mathbb{R}^{2F  } \to \mathbb{R}^H$ and $\Psi: \mathbb{R}^{F + H } \to \mathbb{R}^{F'}$  be Lipschitz continuous with Lipschitz constants $\LipPhi(M_\tr)$ and $\LipPsi(M_\tr)$. Consider a metric-space signal  $f^\pairwisenodes: \chi\times \chi \to \mathbb{R}^F$ with $\|f^\pairwisenodes\|_\infty < \infty$.  Let  $p \in (0,\frac{1}{2(H+1)})$, 
and let $N \in \mathbb{N}$ satisfy (\ref{eq:largeN-new-new}). Suppose that $X_1, \ldots, X_N$ are drawn i.i.d. from $\P$ in $\chi$, and then edges $A_{i,j}\sim \text{Ber}(W(X_i,X_j))$ i.i.d sampled. Then, condition (\ref{eq:HoeffdingsAppOnY_x_new_degree-new}) together with (\ref{eq:lemmab6-1-new-new}) below  are satisfied in probability at least  $1-2(H + 1)p$, 
\begin{equation}
\label{eq:lemmab6-1-new-new}
   \begin{aligned}
& \max_{i,j=1,...,N} \| \Psi\Big(f^\pairwisenodes(\cdot,\cdot), \Mtwo_A\big(\Phi (f^\pairwisenodes,f^\pairwisenodes)\big)(X_i,X_j) \Big) - \Psi \Big(f^\pairwisenodes(\cdot,\cdot) , \Mtwo_W \big(\Phi (f^\pairwisenodes,f^\pairwisenodes) \big) (X_i,X_j) \Big)\|_\infty \\ &  \leq \LipPsi(M_\tr) \Big(4 \frac{\sqrt{2}\sqrt{\log (2N^2/p)}}{\sqrt{N} d_\text{cmin}^2} (\LipPhi(M_\tr) \|f^\pairwisenodes\|_\infty+ \|\Phi(0,0)\|_\infty )\\&+ \frac{2\sqrt{2}(\LipPhi(M_\tr)\|f^\pairwisenodes\|_\infty +  \|\Phi(0, 0)\|_\infty)\sqrt{\log (2N^2/p)}}{d_\text{cmin}\sqrt{N}}\Big),
 \end{aligned} 
\end{equation}
\end{corollary}
\CRupdate{
\begin{proof}
The proof is the same as \citet[Lemma B.6]{maskey2022stability}. The difference comes from the different bound in our \Cref{lemma:TransfMessage-new} and the bound used by \citet[Lemma B.5]{maskey2022stability}.
\end{proof}
} %
\subsection{Proof for \Cref{thm:MainInProb-new2}}

\CRupdate{
Finally, we can prove \Cref{thm:MainInProb-new2}. %
The proof closely follows that of \citet[Theorem B.14]{maskey2022stability}, adapted to our setting. Using the same overall framing as \citet[Theorem B.14]{maskey2022stability}.
}

\thmtwo*

\begin{proof}

In this case, $\|\Phi^{(l)}(0,0)\|_\infty$, $\|\Psi^{(l)}(0,0)\|_\infty$ can be determined by $(G^\tr,\mF^\tr), N^\tr$ if the MPNN $\Theta$ has been trained on the training graph $(G^\tr,\mF^\tr)$.

Following the procedure of \citet[Appendix B.2]{maskey2022stability} with \Cref{thm:TransfMessages-new-new}, we can derive similarly, with probability at least $1- \sum_{l=1}^T(2H_l + 1) p$,
\begin{equation}
    \label{eq:main1-2-pair}
    \begin{aligned}
\deltaAWtwo &\leq \sum_{l=1}^{T}
\LipPsi^{(l)}(M_\tr) \Big(4 \frac{\sqrt{2}\sqrt{\log (2(N^\te)^2/p)}}{\sqrt{N^\te} d_\text{cmin}^2} (\LipPhi^{(l)}(M_\tr) \|{\ftwo}^{(l)}\|_\infty+ \|\Phi^{(l)}(0,0)\|_\infty ) \\&+ \frac{2\sqrt{2}(\LipPhi^{(l)(M_\tr)} \|{\ftwo}^{(l)}\|_\infty+ \|\Phi^{(l)}(0,0)\|_\infty )\sqrt{\log (2(N^\te)^2/p)}}{d_\text{cmin}\sqrt{N^\te}}\Big)
\\&\prod_{l' = l+1}^{T} ((\LipPsi^{(l')}(M_\tr))^2 
  + 
   \frac{8}{d_\text{cmin}^2} (\LipPhi^{(l')}(M_\tr))^2 (\LipPsi^{(l')}(M_\tr))^2)  ,
\end{aligned}
\end{equation}

\CRupdate{
Using the same proof in \citet[Lemma B.9]{maskey2022stability}, we can derive
\[
||{\ftwo}^{(l)}||_{\infty} \leq {B_1^\pairwisenodes}^{(l)}+{B_2^\pairwisenodes}^{(l)}||f||_{\infty},
\]
where ${B_1}^{\pairwisenodes(l)}$, ${B_2}^{\pairwisenodes(l)}$ are independent of $\ftwo$, and \begin{equation}
    \label{eq:B'-pair}
    \begin{aligned}
    {B_1^\pairwisenodes}^{(l+1)} & = \sum_{k=1}^{l+1}  \big(
\LipPsi^{(k)}(M_\tr) \frac{1}{d_\text{cmin}}\|\Phi^{(k)}(0,0)\|_\infty+ \|\Psi^{(k)}(0,0)\|_\infty \big) \\ &\prod_{l' = k+1}^{l+1}  \LipPsi^{(l')}(M_\tr) \big( 1 + \frac{1}{d_\text{cmin}}  \LipPhi^{(l')}(M_\tr) \big) 
    \end{aligned}
\end{equation}
and
\begin{equation}
    \label{eq:B''-pair}
    {B_2^\pairwisenodes}^{(l+1)} = \prod_{k = 1}^{l+1} \LipPsi^{(k)}(M_\tr) \left(1  + \frac{1}{d_\text{cmin}}  \LipPhi^{(k)}(M_\tr) \right).
\end{equation}
}

Now we can decompose the summation in \Cref{eq:main1-2-pair}. First, we defince $C_3$ as
\begin{equation}
\label{eq:c1-pair}
\begin{aligned}
C_3 &= \sum_{l=1}^{T}
\LipPsi^{(l)}(M_\tr) \Big(4 \frac{\sqrt{2}}{ d_\text{cmin}^2} (\LipPhi^{(l)}(M_\tr)B_1^{\pairwisenodes(l)}+ \|\Phi^{(l)}(0,0)\|_\infty ) \\&+ \frac{2\sqrt{2}(\LipPhi^{(l)}(M_\tr) B_1^{\pairwisenodes(l)}+ \|\Phi^{(l)}(0,0)\|_\infty )}{d_\text{cmin}}\Big)
\\&\prod_{l' = l+1}^{T} ((\LipPsi^{(l')}(M_\tr))^2 
  + 
   \frac{8}{d_\text{cmin}^2} (\LipPhi^{(l')}(M_\tr))^2 (\LipPsi^{(l')}(M_\tr))^2) ,
\end{aligned}
\end{equation}
Then we can define $C_4$ as
\begin{equation}
\label{eq:c2-pair}
\begin{aligned}
C_4 &= \sum_{l=1}^{T}
\LipPsi^{(l)}(M_\tr) \Big(4 \frac{\sqrt{2}}{ d_\text{cmin}^2} \LipPhi^{(l)}(M_\tr)B_2^{\pairwisenodes(l)} + \frac{2\sqrt{2}\LipPhi^{(l)}(M_\tr) B_2^{\pairwisenodes(l)}}{d_\text{cmin}}\Big)
\\&\prod_{l' = l+1}^{T} ((\LipPsi^{(l')}(M_\tr))^2 
  + 
   \frac{8}{d_\text{cmin}^2} (\LipPhi^{(l')}(M_\tr))^2 (\LipPsi^{(l')}(M_\tr))^2) ,
\end{aligned}
\end{equation}
It is clear to see we can rewrite \Cref{eq:main1-2-pair} as
\begin{equation}
    \deltaAWtwo
\leq (C_3+C_4\|f^\pairwisenodes\|_\infty)\frac{\sqrt{\log (2(N^\te)^2/p)}}{\sqrt{N^\te}}.
\end{equation}
Thus $C_3$ and $C_4$ depends on $\{\LipPhi^{(l)}(M_\tr)\}_{l=1}^T$ and $\{\LipPsi^{(l)}(M_\tr)\}_{l=1}^T$ and possibly on $(G^\tr,\mF^\tr)$ and $N^\tr$.
\end{proof} %

\end{bibunit}

\end{document}